\documentclass{article}

\usepackage{times}
\usepackage{graphicx} 
\usepackage{subfigure} 

\usepackage{natbib}
\usepackage{xcolor}

\usepackage{algorithm}
\usepackage{algorithmic}

\usepackage{fullpage}
\usepackage{amsmath,amsbsy,amsfonts,amssymb,amsthm}
\usepackage{graphicx}
\usepackage{cases}
\usepackage{subfigure}
\usepackage{natbib}
\usepackage{hyperref}

\newtheorem{theorem}{Theorem}
\newtheorem{lemma}[theorem]{Lemma}
\newtheorem{corollary}[theorem]{Corollary}

\theoremstyle{definition}
\newtheorem{definition}{Definition}

\newcommand{\la}{\langle}
\newcommand{\ra}{\rangle}
\newcommand{\mZ}{\mat{Z}}
\newcommand{\sigstarl}{\sigma^\star_1}
\newcommand{\sigstarr}{\sigma^\star_r}
\newcommand{\Snew}{\S^{\text{opt}}}

\newcommand{\Sopt}{\S^{\text{opt}}}
\newcommand{\cXstar}{\mathcal{X}^\star}

\date{}
\newcommand{\vect}[1]{\ensuremath{\mathbf{#1}}}
\newcommand{\mat}[1]{\ensuremath{\mathbf{#1}}}

\newcommand{\grad}{\nabla}
\newcommand{\hess}{\nabla^2}
\newcommand{\argmin}{\mathop{\rm argmin}}
\newcommand{\argmax}{\mathop{\rm argmax}}

\newcommand{\norm}[1]{\|{#1}\|}
\newcommand{\fnorm}[1]{\|{#1}\|_{\text{F}}}

\newcommand{\tr}{\text{tr}}

\newcommand{\trans}{^{\top}}
\newcommand{\poly}{\text{poly}}

\newcommand{\proj}{\mathcal{P}}

\newcommand{\N}{\mat{N}}
\newcommand{\R}{\mathbb{R}}
\newcommand{\E}{\mathbb{E}}

\newcommand{\A}{\mat{A}}
\newcommand{\B}{\mat{B}}

\renewcommand{\S}{\mat{S}}
\newcommand{\M}{\mat{M}}
\newcommand{\I}{\mat{I}}
\newcommand{\D}{\mat{D}}
\newcommand{\T}{\mat{T}}
\newcommand{\U}{\mat{U}}
\newcommand{\V}{\mat{V}}
\newcommand{\W}{\mat{W}}
\newcommand{\X}{\mat{X}}
\newcommand{\Y}{\mat{Y}}
\newcommand{\mR}{\mat{R}}

\newcommand{\e}{\vect{e}}
\renewcommand{\u}{\vect{u}}
\renewcommand{\v}{\vect{v}}

\newcommand{\x}{\vect{x}}
\newcommand{\y}{\vect{y}}
\newcommand{\z}{\vect{z}}

\renewcommand{\H}{\mathcal{H}}
\newcommand{\G}{\mathcal{G}}

\newcommand{\cn}{\kappa}
\newcommand{\nn}{\nonumber}

\begin{document}

\title{\textbf{No Spurious Local Minima in Nonconvex Low Rank Problems: A Unified Geometric Analysis}}
\author{
Rong Ge\footnote{Duke University. Email: rongge@cs.duke.edu} \and
Chi Jin\footnote{University of California, Berkeley. Email: chijin@cs.berkeley.edu} \and 
Yi Zheng\footnote{Duke University. Email: sheng.zheng@duke.edu} 
}

\maketitle

\begin{abstract} 

In this paper we develop a new framework that captures the common landscape underlying the common non-convex low-rank matrix problems including matrix sensing, matrix completion and robust PCA. In particular, we show for all above problems (including asymmetric cases): 1) all local minima are also globally optimal; 2) no high-order saddle points exists. These results explain why simple algorithms such as stochastic gradient descent have global converge, and efficiently optimize these non-convex objective functions in practice. Our framework connects and simplifies the existing analyses on optimization landscapes for matrix sensing and symmetric matrix completion. The framework naturally leads to new results for asymmetric matrix completion and robust PCA.
\end{abstract} 


\section{Introduction}

Non-convex optimization is one of the most powerful tools in machine learning. Many popular approaches, from traditional ones such as matrix factorization \citep{hotelling1933analysis} to modern deep learning \citep{bengio2009learning} rely on optimizing non-convex functions. In practice, these functions are optimized using simple algorithms such as alternating minimization or gradient descent. Why such simple algorithms work is still a mystery for many important problems.

One way to understand the success of non-convex optimization is to study the optimization landscape: for the objective function, where are the possible locations of global optima, local optima and saddle points. Recently, a line of works showed that several natural problems including tensor decomposition \citep{ge2015escaping}, dictionary learning \citep{sun2015complete1}, matrix sensing \citep{bhojanapalli2016global,park2016non} and matrix completion \citep{ge2016matrix} have well-behaved optimization landscape: all local optima are also globally optimal. Combined with recent results (e.g. \citet{ge2015escaping, carmon2016accelerated, agarwal2016finding, jin2017escape}) that are guaranteed to find a local minimum for many non-convex functions, such problems can be efficiently solved by basic optimization algorithms such as stochastic gradient descent.

In this paper we focus on optimization problems that look for low rank matrices using partial or corrupted observations. Such problems are studied extensively \citep{fazel2002matrix,rennie2005fast,candes2009exact} and has many applications in recommendation systems \citep{koren2009bellkor}, see survey by \citet{davenport2016overview}. 
These optimization problems can be formalized as follows: 
\begin{align}
\min_{\M\in \R^{d_1\times d_2}} &\quad f(\M), \label{eq:convexobj}\\
s.t. & \quad \mbox{rank}(\M) = r. \nonumber
\end{align}
Here $\M$ is an $d_1\times d_2$ matrix and $f$ is a convex 
function of $\M$. The non-convexity of this problem stems from the low rank constraint. 
Several interesting problems, such as matrix sensing \citep{recht2010guaranteed}, matrix completion \citep{candes2009exact} and robust PCA \citep{candes2011robust} can all be framed as optimization problems of this form(see Section~\ref{sec:problems}). 

In practice, \citet{burer2003nonlinear} heuristic is often used \--- replace $\M$ with an explicit low rank representation $\M = \U\V^\top$, where $\U\in \R^{d_1\times r}$ and $\V\in\R^{d_2\times r}$. The new optimization problem becomes
\begin{equation}
\min_{\U\in \R^{d_1\times r},\V\in\R^{d_2\times r}} f(\U\V^\top) + Q(\U,\V). \label{eq:asymmetricobj}
\end{equation}
Here $Q(\U,\V)$ is a (optional) regularizer. 
Despite the objective being non-convex, for all the problems mentioned above, simple iterative updates from random or even arbitrary initial point find the optimal solution in practice. 
It is then natural to ask: {\bf Can we characterize the similarities between the optimization landscape of these problems?} We show this is indeed possible:
%
%




\begin{theorem}[informal] The objective function of matrix sensing, matrix completion and robust PCA have similar optimization landscape. In particular, for all these problems, 1) all local minima are also globally optimal; 2) any saddle point has at least one strictly negative eigenvalue in its Hessian.
\end{theorem}

More precise theorem statements appear in Section~\ref{sec:problems}. Note that there were several cases (matrix sensing \citep{bhojanapalli2016global,park2016non}, symmetric matrix completion \citep{ge2016matrix}) where similar results on the optimization landscape were known. However the techniques in previous works are tailored to the specific problems and hard to generalize. Our framework captures and simplifies all these previous results, and also gives new results on asymmetric matrix completion and robust PCA. 

The key observation in our analysis is that for matrix sensing, matrix completion, and robust PCA (when fixing sparse estimate),
function $f$ (in Equation \eqref{eq:convexobj}) is a quadratic function over the matrix $\M$. Hence the Hessian $\H$ of $f$ with respect to $\M$ is a constant. 
More importantly, the Hessian $\H$ in all above problems has similar properties (that it approximately preserves norm, similar to the RIP properties used in matrix sensing \citep{recht2010guaranteed}), which allows their optimization landscapes to be characterized in a unified way.
Specifically, our framework gives principled way of defining a {\em direction of improvement} for all points that are not globally optimal.


Another crucial property of our framework is the interaction between the regularizer and the Hessian $\H$. Intuitively, the regularizer makes sure the solution is in a nice region $\mathcal{B}$ (e.g. set of incoherent matrices for matrix completion), and only within $\mathcal{B}$ the Hessian has the norm preserving property. On the other hand, regularizer should not be too large to severely distort the landscape. This interaction is crucial for matrix completion, and is also very useful in handling noise and perturbations. 
In Section~\ref{sec:symmetric}, we discuss ideas required to apply this framework to matrix sensing, matrix completion and robust PCA.

Using this framework, we also give a way to {\em reduce} asymmetric matrix problems to symmetric PSD problems (where the desired matrix is of the form $\U\U^\top$). See Section~\ref{sec:asymmetric} for more details.

In addition to the results of no spurious local minima, our framework also implies that any saddle point has at least one strictly negative eigenvalue in its Hessian. Formally, we proved all above problems satisfy a robust version of this claim --- strict saddle property (see Definition~\ref{def:strict_saddle}), which is one of crucial sufficient conditions to admit efficient optimization algorithms, and thus following corollary (see Section \ref{sec:runtime} for more details).
\begin{corollary}[informal]\label{cor:runtime}
For matrix sensing, matrix completion and robust PCA, simple local search algorithms can find the desired low rank matrix $\U\V^\top = \M^\star$ from an arbitrary starting point in polynomial time with high probability.
\end{corollary}

For simplicity, we present most results in the noiseless setting, but our results can also be generalized to handle noise. As an example, we show how to do this for matrix sensing in Section~\ref{sec:noise}.

\subsection{Related Works}
The landscape of low rank matrix problems have recently received a lot of attention. \citet{ge2016matrix} showed symmetric matrix completion has no spurious local minimum. At the same time, \citet{bhojanapalli2016global} proved similar result for symmetric matrix sensing. \citet{park2016non} extended the matrix sensing result to asymmetric case. All of these works guarantee global convergence to the correct solution.

There has been a lot of work on the local convergence analysis for various algorithms and problems. For matrix sensing or matrix completion, the works
\citep{keshavan2010matrix,keshavan2010matrixnoisy, hardt2014fast,hardt2014understanding,jain2013low,chen2015fast,sun2015guaranteed, zhao2015nonconvex,zheng2016convergence,tu2015low} showed that given a good enough initialization, many simple local search algorithms, including gradient descent and alternating least squares, succeed. Particularly, several works (e.g. \citet{sun2015guaranteed, zheng2016convergence}) accomplished this by showing a geometric property which is very similar to strong convexity holds in the neighborhood of optimal solution. For robust PCA, there are also many analysis for local convergence \citep{lin2010augmented,netrapalli2014non,yi2016fast,zhang2017nonconvex}.

Several works also try to unify the analysis for similar problems. \citet{bhojanapalli2015dropping} gave a framework for local analysis for these low rank problems. \citet{belkin2014basis} showed a framework of learning basis functions, which generalizes tensor decompositions. Their techniques imply the optimization landscape for all such problems are very similar. For problems looking for a symmetric PSD matrix, 
\citet{li2016nonconvex} showed for objective similar to \eqref{eq:asymmetricobj} (but in the symmetric setting), restricted smoothness/strong convexity on the function $f$ suffices for local analysis. However, their framework does not address the interaction between regularizer and the function $f$, hence cannot be directly applied to problems such as matrix completion or robust PCA.


\paragraph{Organization} We will first introduce notations and basic optimality conditions in Section~\ref{sec:prelim}. Then Section~\ref{sec:problems} introduces the problems and our results. For simplicity, we present our framework for the symmetric case in Section~\ref{sec:symmetric}, and briefly discuss how to reduce asymmetric problem to symmetric problem in Section~\ref{sec:asymmetric}. We discuss how our geometric result implies efficient algorithms in Section~\ref{sec:runtime}. We then show how our geometric results imply fast runtime of popular local search algorithms in Section~\ref{sec:runtime}. For clean presentation, many proofs are deferred to appendix .



\section{Preliminaries}
\label{sec:prelim}
In this section we introduce notations and basic optimality conditions. 

\subsection{Notations}

We use bold letters for matrices and vectors. For a vector $\v$ we use $\|\v\|$ to denote its $\ell_2$ norm. For a matrix $\M$ we use $\|\M\|$ to denote its spectral norm, and $\|\M\|_F$ to denote its Frobenius norm. For vectors we use $\la \u,\v \ra$ to denote inner-product, and for matrices we use $\la \M,\N\ra = \sum_{i,j} \M_{ij}\N_{ij}$ to denote the trace of $\M\N\trans$. 
We will always use $\M^\star$ to denote the optimal low rank solution. Further, we use $\sigstarl$ to denote its largest singular value, $\sigstarr$ to denote its $r$-th singular value and $\cn^\star = \sigstarl/\sigstarr$ be the condition number.

We use $\nabla f$ to denote the gradient and $\nabla^2 f$ to denote its Hessian. Since function $f$ can often be applied to both $\M$ (as in \eqref{eq:convexobj}) and $\U,\V$ (as in \eqref{eq:asymmetricobj}), we use $\nabla f(\M)$ to denote gradient with respect to $\M$ and $\nabla f(\U,\V)$ to denote gradient with respect to $\U,\V$. Similar notation is used for Hessian.
The Hessian $\nabla^2 f(\M)$ is a crucial object in our framework. It can be interpreted as a linear operator on matrices. This linear operator can be viewed as a $d_1d_2\times d_1d_2$ matrix (or $\binom{d+1}{2} \times \binom{d+1}{2}$ matrix in the symmetric case) that applies to the vectorized version of matrices. We use the notation $\M:\mathcal{H}:\N$ to denote the quadratic form $\la \M, \mathcal{H}(\N)\ra$. Similarly, the Hessian of objective \eqref{eq:asymmetricobj} is a linear operator on a pair of matrices $\U,\V$, which we usually denote as $\nabla^2 f(\U,\V)$.

\subsection{Optimality Conditions}

\paragraph{Local Optimality}
Suppose we are optimizing a function $f(\x)$ with no constraints on $\x$. In order for a point $\x$ to be a local minimum, it must satisfy the first and second order necessary conditions. That is, we must have $\nabla f(\x) = 0$ and $\nabla^2 f(\x) \succeq 0$. 

\begin{definition}[Optimality Condition]\label{def:optimality}
Suppose $\x$ is a \textbf{local minimum} of $f(\x)$, then we have
$$\nabla f(\x) = 0, \quad \nabla^2 f(\x)\succeq 0.$$
\end{definition}

Intuitively, if one of these conditions is violated, then it is possible to find a direction that decreases the function value. \cite{ge2015escaping} characterized the following {\em strict-saddle} property, which is a quantitative version of the optimality conditions, and  can lead to efficient algorithms to find local minima.

\begin{definition}\label{def:strict_saddle}
We say function $f(\cdot)$ is $(\theta, \gamma, \zeta)$-\textbf{strict saddle}. That is, for any $\x$, at least one of followings holds:
\begin{enumerate}\itemsep=0pt
\item $\norm{\grad f(\x)} \ge \theta$.
\item $\lambda_{\min}(\hess f(\x)) \le -\gamma$.
\item $\x$ is $\zeta$-close to $\cXstar$ \--- the set of local minima.
\end{enumerate}
\end{definition}

Intuitively, this definition says for any point $\x$, it either violates one of the optimality conditions significantly (first two cases), or is close to a local minima. Note that $\zeta$ and $\theta$ are often closely related. For a function with strict-saddle property, it is possible to efficiently find a point near a local minimum.

\paragraph{Local vs. Global} However, of course finding a local minimum is not sufficient in many case. In this paper we are also going to prove that all local minima are also globally optimal, and they correspond to the desired solutions.


\section{Low Rank Problems and Our Results}
\label{sec:problems}

In this section we introduce matrix sensing, matrix completion and robust PCA. For each problem we give the results obtained by our framework. The proof ideas are illustrated later in Sections~\ref{sec:symmetric} and \ref{sec:asymmetric}.

\subsection{Matrix Sensing}

Matrix sensing \citep{recht2010guaranteed} is a generalization of compressed sensing \citep{candes2006stable}. In the matrix sensing problem, there is an unknown low rank matrix $\M^\star \in \R^{d_1\times d_2}$. We make linear observations on this matrix: let $\A_1, \A_2,..., \A_m \in \R^{d_1\times d_2}$ be $m$ sensing matrices, the algorithm is given $\{\A_i\}$'s and the corresponding $b_i = \la \A_i, \M^\star\ra$. The goal is now to find the unknown matrix $\M^\star$. In order to find $\M^\star$, we need to solve the following nonconvex optimization problem
\begin{align*}
\min_{\M\in \R^{d_1\times d_2},\mbox{rank}(\M) = r} &\quad f(\M) = \frac{1}{2m}\sum_{i=1}^m (\la \M,\A_i\ra -b_i)^2.
\end{align*}

We can transform this constraint problem to an unconstraint problem by expressing $\M$ as $\M= \U\V^\top$ where $\U\in \R^{d_1\times r}$ and $\V\in \R^{d_2\times r}$. We also need an additional regularizer (common for all asymmetric problems):
\begin{equation}\label{eq:sensing-asymmetric}
\min_{\U,\V} \quad \frac{1}{2m}\sum_{i=1}^m (\la \U\V^\top,\A_i\ra -b_i)^2 + \frac{1}{8} \|\U^\top \U - \V^\top \V\|_F^2.
\end{equation}

The regularizer has been widely used in previous works \citep{zheng2016convergence, park2016non}. In Section~\ref{sec:asymmetric} we show how this regularizer can be viewed as a way to deal with the additional invariants in asymmetric case, and reduce the asymmetric case to the symmetric case.
A crucial concept in standard sensing literature is Restrict Isometry Property (RIP), which is defined as follows:

%
%

\begin{definition}A group of sensing matrices $\{\A_1,..,\A_m\}$ satisfies the $(r,\delta)$-RIP condition, if for every matrix $\M$ of rank at most $r$, 
$$
(1-\delta)\|\M\|_F^2\le \frac{1}{m} \sum_{i=1}^m \la \A_i, \M\ra^2 \le (1+\delta)\|\M\|_F^2.
$$\label{def:rip}
\end{definition} 

Intuitively, RIP says operator $\frac{1}{m}\sum_{i=1}^m \la \A_i, \cdot\ra^2$ approximately perserve norms for all low rank matrices.
When the sensing matrices are chosen to be i.i.d. matrices with independent Gaussian entries, if $m \ge c (d_1+d_2)r$ for large enough constant $c$, the sensing matrices satisfy the $(2r,\frac{1}{20})$-RIP condition \citep{candes2011tight}. Using our framework we can show:

\begin{theorem}\label{thm:sensing-main}
When measurements $\{\A_i\}$ satisfy $(2r, \frac{1}{20})$-RIP, for matrix sensing objective \eqref{eq:sensing-asymmetric} we have 1) all local minima satisfy $\U\V^\top = \M^\star$ 2) the function is $(\epsilon, \Omega(\sigstarr), O(\frac{\epsilon}{\sigstarr}))$-strict saddle.
\end{theorem}
This in particular says 1) no spurious local minima existsl; 2) whenever at some point $(\U, \V)$ so that the gradient is small and the Hessian does not have significant negative eigenvalue, then the distance to global optimal (see Definition~\ref{def:delta} and Definition~\ref{def:asymmetricquantities}) is guaranteed to be small. Such a point can be found efficiently (see Section~\ref{sec:runtime}).

\subsection{Matrix Completion}

Matrix completion is a popular technique in recommendation systems and collaborative filtering \citep{koren2009bellkor,rennie2005fast}. In this problem, again we have an unknown low rank matrix $\M^\star$. We observe each entry of the matrix $\M^\star$ independently with probability $p$. Let $\Omega \subset [d_1]\times [d_2]$ be a set of observed entries. For any matrix $\M$, we use $\M_\Omega$ to denote the matrix whose entries outside of $\Omega$ are set to 0. That is, $[\M_\Omega]_{i,j} = \M_{i,j}$ if $(i,j)\in \Omega$, and $[\M_\Omega]_{i,j} = 0$ otherwise. We further use $\|\M\|_\Omega$ to denote $\|\M_\Omega\|_F$. Matrix completion can be viewed as a special case of matrix sensing, where the sensing matrices only have one nonzero entry. However such matrices do not satisfy the RIP condition.

In order to solve matrix completion, we try to optimize the following:
\begin{align*}
\min_{\M\in \R^{d_1\times d_2},\mbox{rank}(\M)=r} &\quad \frac{1}{2p}\|\M - \M^\star\|_\Omega^2.
\end{align*}

A well-known problem in matrix completion is that when the true matrix $\M^\star$ is very sparse, then we are very likely to observe only $0$ entries, and has no chance to learn the other entries of $\M^\star$. To avoid this case, previous works have assumed following {\em incoherence} condition:

\begin{definition}\label{def:incoherence}
A rank $r$ matrix $\M\in \R^{d_1\times d_2}$ is $\mu$-incoherent, if for the rank-$r$ SVD $\X \D\Y\trans$ of $\M$, we have for all $i\in[d_1],j\in[d_2]$
$$
\|\e_i^\top \X\| \le \sqrt{\mu r/d_2}, \quad \|\e_j^\top \Y\| \le \sqrt{\mu r/d_1}.
$$
\end{definition}
We assume the unknown optimal low rank matrix $\M^\star$ is $\mu$-incoherent.

In the non-convex program, we try to make sure the decomposition $\U\V^\top$ is also incoherent by adding a regularizer 
\begin{align*}Q(\U,\V) &= \lambda_1 \sum_{i=1}^{d_1} (\norm{\e_i\trans \U} - \alpha_1)^4_{+}+\lambda_2 \sum_{j=1}^{d_2} (\norm{\e_j\trans \V} - \alpha_2)^4_{+}.
\end{align*}
Here $\lambda_1,\lambda_2, \alpha_1,\alpha_2$ are parameters that we choose later, $(x)_+ = \max\{x,0\}$. Using this regularizer, we can now transform the objective function to the unconstraint form
\begin{align}
\min_{\U,\V} &\quad \frac{1}{2p}\|\U\V^\top - \M^\star\|_\Omega^2 
 + \frac{1}{8}\|\U^\top \U - \V^\top \V\|_F^2 + Q(\U,\V).
\label{eq:completion-asymmetric}
\end{align}

Using the framework, we can show following:

\begin{theorem}\label{thm:main_mc_asym}
Let $d = \max\{d_1, d_2\}$, when sample rate $p \ge \Omega(\frac{\mu^4 r^6 (\cn^\star)^6 \log d}{\min\{d_1, d_2\}})$, choose $\alpha_1^2 = \Theta(\frac{\mu r\sigstarl}{d_1}), \alpha_2^2 = \Theta(\frac{\mu r\sigstarl}{d_2})$ and $\lambda_1 = \Theta(\frac{d_1}{\mu r \cn^\star}), \lambda_2 = \Theta(\frac{d_2}{\mu r \cn^\star})$. With probability at least $1-1/\poly(d)$, for Objective Function \eqref{eq:completion-asymmetric} we have
1) all local minima satisfy $\U\V^\top = \M^\star$ 2) The objective is $(\epsilon, \Omega(\sigstarr), O(\frac{\epsilon}{\sigstarr}))$-strict saddle for polynomially small $\epsilon$.
\end{theorem}

\subsection{Robust PCA}

Robust PCA \citep{candes2011robust} is a generalization to the standard Principled Component Analysis. In Robust PCA, we are given an observation matrix $\M_o$, which is an true underlying matrix $\M^\star$ corrupted by a sparse noise $\S^\star$ ($\M_o = \M^\star + \S^\star$). In some sense the goal is to decompose the matrix $\M$ into these two components. There are many models on how many entries can be perturbed, and how they are distributed. In this paper we work in the setting where $\M^\star$ is $\mu$-incoherent, and the rows/columns of $\S^\star$ can have at most $\alpha$-fraction non-zero entries.

In order to express robust PCA as an optimization problem, we need constraints on both $\M$ and $\S$:
\begin{align}
\min &~~ \frac{1}{2}\|\M + \S - \M_o\|_F^2.\\
s.t. &~~ \mbox{rank}(\M) \le r, \S\mbox{ is sparse}.\nonumber
\end{align}

There can be several ways to specify the sparsity of $\S$.
In this paper we restrict attention to the following set:
\begin{align*}
&\mathcal{S}_{\alpha} = \left\{ \S \in \R^{d_1\times d_2} ~|~ \S \text{~has at most $\alpha$-fraction non-zero entries each column/row, and~} \norm{\S}_\infty \le 2\frac{\mu r \sigstarl}{\sqrt{d_1d_2}}\right\}.
\end{align*}

Assuming the true sparse matrix $\S^\star$ is in $\mathcal{S}_{\alpha}$.
Note that the infinite norm requirement on $\S^\star$ is without loss of generality, because by incoherence $\M^\star$ cannot have entries with absolute value more than $\frac{\mu r \sigstarl}{\sqrt{d_1d_2}}$. Any entry larger than that is obviously in the support of $\S^\star$ and can be truncated.

In objective function, we allow $\S$ to be $\gamma$ times denser (in $\mathcal{S}_{\gamma\alpha}$) where $\gamma$ is a parameter we choose later. Now the constraint optimization problem can be tranformed to the unconstraint problem
\begin{align}
\min_{\U,\V}  ~~ & f(\U,\V) + \frac{1}{8}\|\U^\top \U - \V^\top \V\|_F^2,\label{eq:rpca-nonconvex}\\
&f(\U,\V): = \min_{\S\in\mathcal{S}_{\gamma\alpha}}\frac{1}{2}\|\U\V^\top + \S - \M_o\|_F^2.\nonumber
\end{align}

Of course, we can also think of this as a joint minimization problem of $\U,\V, \S$. However we choose to present it this way in order to allow extension of the strict-saddle condition. Since $f(\U, \V)$ is not twice-differetiable w.r.t $\U, \V$, it does not admit Hessian matrix, so we use the following generalized version of strict-saddle

\begin{definition}\label{def:pseudo_strict_saddle}
We say function $f(\cdot)$ is $(\theta, \gamma, \zeta)$-\textbf{pseudo strict saddle} if for any $\x$, at least one of followings holds:
\begin{enumerate}
\item $\norm{\grad f(\x)} \ge \theta$.
\item $\exists g_{\x}(\cdot)$ so that $\forall \y, g_{\x}(\y) \ge f(\y)$; $g_{\x}(\x) = f(\x)$;
$\lambda_{\min}(\hess g_{\x}(\x)) \le -\gamma$.
\item $\x$ is $\zeta$-close to $\cXstar$ \--- the set of local minima.
\end{enumerate}
\end{definition}

Note that in this definition, the upperbound in 2 can be viewed as similar to the idea of subgradient. For functions with non-differentiable points, subgradient is defined so that it still offers a lowerbound for the function. In our case this is very similar \--- although Hessian is not defined, we can use a smooth function that upperbounds the current function (upper-bound is required for minimization). In the case of robust PCA the upperbound is obtained by a fixed $\S$. Using this formalization we can prove

\begin{theorem}\label{thm:main_rpca_asym}
There is an absolute constant $c>0$, if $\gamma > c$, and $\gamma\alpha\cdot\mu r\cdot (\cn^\star)^5 \le \frac{1}{c}$ holds, for objective function Eq.\eqref{eq:rpca-nonconvex} we have 1) all local minima satisfies $\U\V\trans = \M^\star$; 2) objective function is 
$(\epsilon, \Omega(\sigstarr), O(\frac{\epsilon \sqrt{\cn^\star}}{\sigstarr}))$-pseudo strict saddle for polynomially small $\epsilon$.
\end{theorem}


\section{Framework for Symmetric Positive Definite Problems}
\label{sec:symmetric}

In this section we describe our framework in the simpler setting where the desired matrix is positive semidefinite. In particular, suppose the true matrix $\M^\star$ we are looking for can be written as $\M^\star = \U^\star (\U^\star)^\top$ where $\U^\star\in \R^{d\times r}$. For objective functions that is quadratic over $\M$, we denote its Hessian as $\H$ and we can write the objective as
\begin{align}
\min_{\M\in \R^{d\times d}_{\text{sym}},\mbox{rank}(\M)=r} &\quad \frac{1}{2}(\M-\M^\star):\mathcal{H}:(\M-\M^\star), \label{eq:symmetricconvexobj}
\end{align}
We call this objective function $f(\M)$. Via Burer-Monteiro factorization, the corresponding unconstraint optimization problem, with regularization $Q$ can be written as
\begin{equation}
\min_{\U\in \R^{n \times r}} \frac{1}{2}(\U\U^\top - \M^\star):\mathcal{H}:(\U\U^\top - \M^\star) + Q(\U).\label{eq:symmetricnonconvexobj}
\end{equation}

In this section, we also denote $f(\U)$ as objective function with respect to parameter $\U$, abuse the notation of $f(\M)$ previously defined over $\M$.

\paragraph{Direction of Improvement}
The optimality condition (Definition~\ref{def:optimality}) implies if the gradient is non-zero, or if we can find a negative direction of the Hessian (that is a direction $\v$, so that $\v\trans \hess f(\x) \v <0$), then the point is not a local minimum. A common technique in characterizing the optimization landscape is therefore trying to explicitly find this negative direction. We call this the direction of improvement. Different works \citep{bhojanapalli2016global,ge2016matrix} have chosen very different directions of improvement.

In our framework, we show it suffices to choose a single direction $\Delta$ as the direction of improvement. Intuitively, this direction should bring us close to the true solution $\U^\star$ from the current point $\U$. Due to rotational symmetry ($\U$ and $\U\mR$ behave the same for the objective if $\mR$ is a rotation matrix), we need to carefully define the difference between $\U$ and $\U^\star$.

\begin{definition}\label{def:delta} Given matrices $\U,\U^\star \in \R^{d\times r}$, define their difference $\Delta = \U - \U^\star\mR$, where $\mR\in \R^{r\times r}$ is chosen as $
\mR = \argmin_{\mZ^\top \mZ = \mZ\mZ^\top = \I}\|\U-\U^\star\mZ\|_F^2.
$
\end{definition}

Note that this definition tries to ``align'' $\U$ and $\U^\star$ before taking their difference, and therefore is invariant under rotations. In particular, this definition has the nice property that as long as $\M = \U\U^\top$ is close to $\M^\star = \U^\star(\U^\star)^\top$, we have $\Delta$ is small (we defer the proof to Appendix):

\begin{lemma}\label{lem:bound}
Given matrices $\U,\U^\star \in \R^{d\times r}$, let $\M = \U\U^\top$ and $\M^\star = \U^\star(\U^\star)^\top$, and let $\Delta$ be defined as in Definition~\ref{def:delta}, then we have $\|\Delta\Delta^\top\|_F^2 \le 2\|\M - \M^\star\|_F^2$, and $\sigstarr\|\Delta\|_F^2\le \frac{1}{2(\sqrt{2}-1)}\|\M-\M^\star\|_F^2$.
\end{lemma} 

Now we can state the main Lemma:

\begin{lemma}[Main]\label{lem:main} For the objective \eqref{eq:symmetricnonconvexobj}, let $\Delta$ be defined as in Definition~\ref{def:delta} and $\M = \U\U\trans$. Then, for any $\U\in \R^{d\times r}$, we have
\begin{align}
\quad\Delta : \hess f(\U) :\Delta = &  \Delta\Delta\trans:\H:\Delta\Delta\trans - 3(\M - \M^\star):\H:(\M - \M^\star)\nonumber\\ 
&+ 4\la \grad f(\U), \Delta\ra
+ [\Delta : \hess Q(\U) : \Delta - 4 \la \grad Q(\U), \Delta \ra]\label{eq:main}
\end{align}\end{lemma}

To see why this lemma is useful, let us look at the simplest case where $Q(\U) = 0$ and $\mathcal{H}$ is identity. In this case, if gradient is zero, by Eq. \eqref{eq:main} 
\begin{align*}
&\Delta : \hess f(\U) :\Delta= \|\Delta\Delta\|_F^2 - 3\|\M - \M^\star\|_F^2
\end{align*}

By Lemma~\ref{lem:bound} this is no more than $-\|\M-\M^\star\|_F^2$. 
Therefore, all stationary point with $\M \ne \M^*$ must be saddle points, and we immediately conclude all local minimum satisfies $\U\U\trans = \M^\star$!

\paragraph{Interaction with Regularizer} For problems such as matrix completion, the Hessian $\H$ does not preserve the norm for all low rank matrices. In these cases we need to use additional regularizer. In particular, conceptually we need the following steps:

\begin{enumerate}
\item Show that the regularizer $Q$ ensures for any $\U$ such that $\nabla f(\U) = 0$, $\U \in \mathcal{B}$ for some set $\mathcal{B}$.
\item Show that whenever $\U\in \mathcal{B}$, the Hessian operator $\H$ behaves similarly as identity: for some $c>0$ we have: 
$
\Delta\Delta\trans:\H:\Delta\Delta\trans - 3(\M - \M^\star):\H:(\M - \M^\star)
< -c\|\Delta\|_F^2.
$
\item Show that the regularizer does not contribute a large positive term to $\Delta : \hess f(\U) :\Delta$. This means we show an upperbound for $
4\la \grad f(\U), \Delta\ra + [\Delta : \hess Q(\U) : \Delta - 4 \la \grad Q(\U), \Delta \ra]. 
$
\end{enumerate}

Interestingly, these steps are not just useful for handling regularizers. Any deviation to the original model (such as noise, or if the optimal matrix is not exactly low rank) can be viewed as an additional ``regularizer'' function $Q(\U)$ and argued in the same framework. See e.g. Section~\ref{sec:noise}.

\subsection{Matrix Sensing}

Matrix sensing is the ideal setting for this framework. For symmetric matrix sensing, the objective function is
\begin{equation}\label{eq:sensing-symmetric}
\min_{\U\in\R^{d\times r}} \frac{1}{2m}\sum_{i=1}^m (\la \A_i, \U\U^\top\ra - b_i)^2.
\end{equation}
Recall that matrices $\{\A_i:i=1,2,...,m\}$ are known sensing matrices, and $b_i = \la \A_i,\M^\star \ra$ is the result of $i$-th observation. The intended solution is the unknown low rank matrix $\M^\star = \U^\star(\U^\star)^\top$. For any low rank matrix $\M$, the Hessian operator satisfies
$$
\M:\H:\M = \sum_{i=1}^m \la \A_i, \M\ra^2.
$$
Therefore if the sensing matrices satisfy the RIP property (Definition~\ref{def:rip}), the Hessian operator is close to identity for all low rank matrices! In the symmetric case there is no regularizer, so the landscape for symmetric matrix sensing follows immediately from our main Lemma~\ref{lem:main}.

\begin{theorem}\label{thm:sensing-symmetric-main}
When measurement $\{\A_i\}$ satisfies $(2r, \frac{1}{10})$-RIP, for matrix sensing objective \eqref{eq:sensing-symmetric} we have 1) all local minima $\U$ satisfy $\U\U\trans = \M^\star$; 2) the function is $(\epsilon, \Omega(\sigstarr), O(\frac{\epsilon}{\sigstarr}))$-strict saddle.
\end{theorem}

\begin{proof}
For point $\U$ with small gradient satisfying $\fnorm{\grad f(\U)} \le \epsilon$, by $(2r, \delta_{2r})$-RIP property:
\begin{align*}
\Delta : \hess f(\U) :\Delta
=&  \Delta\Delta\trans:\H:\Delta\Delta\trans - 3(\M - \M^\star):\H:(\M - \M^\star) +4\la \grad f(\U), \Delta\ra \\
\le & (1+\delta_{2r})\fnorm{\Delta\Delta\trans}^2 - 3(1-\delta_{2r})\fnorm{\M-\M^\star}^2
+ 4\epsilon\fnorm{\Delta}\\
\le &-(1-5\delta_{2r})\fnorm{\M-\M^\star}^2 + 4\epsilon\fnorm{\Delta} 
\\
\le &-0.4\sigstarr\fnorm{\Delta}^2+ 4\epsilon\fnorm{\Delta}
\end{align*}
The second last inequality is due to Lemma \ref{lem:bound} that $\fnorm{\Delta\Delta\trans}^2 \le 2 \fnorm{\M-\M^\star}^2$, and last inequality is due to $\delta_{2r} = \frac{1}{10}$ and second part of Lemma \ref{lem:bound}. This means if $\U$ is not close to $\U^\star$, that is, if $\fnorm{\Delta} \ge \frac{20\epsilon}{\sigstarr}$,
we have $\Delta : \hess f(\U) :\Delta \le -0.2\sigstarr \fnorm{\Delta}^2$. This proves 
$(\epsilon, 0.2 \sigstarr, \frac{20\epsilon}{\sigstarr})$-strict saddle property. Take $\epsilon =0$, we know all stationary points with $\fnorm{\Delta} \neq 0$ are saddle points. This means all local minima are global minima (satisfying $\U\U\trans = \M^\star$), which finishes the proof.
\end{proof}

\subsection{Matrix Completion}

For matrix completion, we need to ensure the incoherence condition (Definition~\ref{def:incoherence}). In order to do that, we add a regularizer $Q(\U)$ that penalize the objective function when some row of $\U$ is too large. We choose the same regularizer as \cite{ge2016matrix}: $Q(\U) = \lambda \sum_{i=1}^d (\|\U_i\|-\alpha)_+^4$. The objective is then
\begin{equation}
\min_{\U\in \R^{d\times r}} \frac{1}{2p}\|\M^\star - \U\U^\top\|_\Omega^2 + Q(\U).\label{eq:completion-symmetric}
\end{equation}

 Using our framework, we first need to show that the regularizer ensures all rows of $\U$ are small (step 1). 

\begin{lemma}\label{lem:incoherence}
There exists an absolute constant $c$, when sample rate $p \ge \Omega(\frac{\mu r}{d} \log d)$, $\alpha^2 = \Theta(\frac{\mu r\sigstarl}{d})$ and $\lambda = \Theta(\frac{d}{\mu r \cn^\star})$, we have for any points $\U$ with $\fnorm{\grad f (\U)} \le \epsilon$ for polynomially small $\epsilon$, with probability at least $1-1/\poly(d)$:
\begin{align*}
\max_i\norm{\e_{i}\trans \U}^2 \le O\left(\frac{(\mu r)^{1.5} \cn^\star \sigstarl}{d}\right)
\end{align*}
\end{lemma}

This is a slightly stronger version of Lemma 4.7 in \cite{ge2016matrix}. Next we show under this regularizer, we can still select the direction $\Delta$, and the first part of Equation~\eqref{eq:main} is significantly negative when $\Delta$ is large (step 2):

\begin{lemma} \label{lem:step2_mc}
When sample rate $p \ge \Omega(\frac{\mu^3 r^4 (\cn^\star)^4 \log d}{d})$, by choosing $\alpha^2 = \Theta(\frac{\mu r\sigstarl}{d})$ and $\lambda = \Theta(\frac{d}{\mu r \cn^\star})$ with probability at least $1-1/\poly(d)$, for all $\U$ with $\fnorm{\grad f (\U)} \le \epsilon$ for polynomially small $\epsilon$ we have
\begin{equation*}
\Delta\Delta\trans:\H:\Delta\Delta\trans - 3(\M - \M^\star):\H:(\M - \M^\star) \le -0.3 \sigstarr \fnorm{\Delta}^2
\end{equation*}
\end{lemma}

This lemma follows from several standard concentration inequalities, and is made possible because of the incoherence bound we proved in the previous lemma. 

Finally we show the additional regularizer related term in Equation \eqref{eq:main} is bounded (step 3).

\begin{lemma}\label{lem:reg_mc}
By choosing $\alpha^2 = \Theta(\frac{\mu r\sigstarl}{d}) $ and $\lambda \alpha^2 \le  O(\sigstarr)$, we have:
\begin{equation*}
\frac{1}{4}[\Delta : \hess Q(\U) : \Delta - 4 \la \grad Q(\U), \Delta \ra] \le 0.1\sigstarr \fnorm{\Delta}^2 
\end{equation*}
\end{lemma}

Combining these three lemmas, it is easy to see
\begin{theorem}\label{thm:main_mc}
When sample rate $p \ge \Omega(\frac{\mu^3 r^4 (\cn^\star)^4 \log d}{d})$, by choosing $\alpha^2 = \Theta(\frac{\mu r\sigstarl}{d})$ and $\lambda = \Theta(\frac{d}{\mu r \cn^\star})$.
Then with probability at least $1-1/\poly(d)$, for matrix completion objective \eqref{eq:completion-symmetric} we have
1) all local minima satisfy $\U\U^\top = \M^\star$ 2) the function is $(\epsilon, \Omega(\sigstarr),O( \frac{\epsilon}{\sigstarr}))$-strict saddle
for polynomially small $\epsilon$. 
\end{theorem}

Notice that our proof is different from \cite{ge2016matrix}, as we focus on the direction $\Delta$ for both first and second order conditions while they need to select different directions for the Hessian. The framework allowed us to get a simpler proof, generalize to asymmetric case and also improved the dependencies on rank.

\subsection{Robust PCA}

In the robust PCA problem, for any given matrix $\M$ the objective function try to find the optimal sparse perturbation $\S$. In the symmetric PSD case, recall we observe $\M_o = \M^\star+\S^\star$, we define the set of sparse matrices to be
\begin{align*}
\mathcal{S}_{\alpha} = \left\{ \S \in \R^{d\times d} ~|~ \S \text{~has at most $\alpha$-fraction non-zero entries each column/row, and~} \norm{\S}_\infty \le 2\frac{\mu r \sigstarl}{d}\right\}.
\end{align*}
Note the projection onto set $\mathcal{S}_{\alpha}$ be computed in polynomial time (using a max flow algorithm).

We assume $\S^\star\in \mathcal{S}_\alpha$, the objective can be written as
\begin{align}
\min_{\U} \quad & f(\U), \quad\quad\quad \text{where} \quad
f(\U): = \min_{\S\in\mathcal{S}_{\gamma\alpha}}\frac{1}{2}\|\U\U^\top + \S - \M_o\|_F^2.\label{eq:rpca-symmetric}
\end{align}
Here $\gamma$ is a slack parameter that we choose later. 

Note that now the objective function $f(\U)$ is not quadratic, so we cannot use the framework directly. However, if we fix $\S$, then $f_{\S}(\U) := \frac{1}{2}\|\U\U^\top + \S - \M_o\|_F^2$ is a quadratic function with Hessian equal to identity. We can still apply our framework to this function. In this case, since the Hessian is identity for all matrices, we can skip the first step. The problem becomes a matrix factorization problem:
\begin{equation}
\min_{\U\in \R^{d\times r}} \frac{1}{2}\|\A - \U\U^\top\|_F^2.
\end{equation}
The difference here is that the matrix $\A$ (which is $\M^\star+\S^\star -\S$) is not equal to $\M^\star$ and is in general not low rank. We can use the framework to analyze this problem (and treat the residue $\A-\M^\star$ as the ``regularizer'' $Q(\U)$).

\begin{lemma}\label{lem:mf_strictsaddle}
Let $\A \in \R^{d\times d}$ be a symmetric PSD matrix, and matrix factorization objective to be:
\begin{equation*}
f(\U) = \fnorm{\U\U\trans - \A}^2
\end{equation*}
where $\sigma_r(\A)\ge 15\sigma_{r+1}(\A)$.
then 1) all local minima satisfies $\U\U\trans = \proj_r(\A)$ (best rank-$r$ approximation),
2) objective is $(\epsilon, \Omega(\sigstarr), O(\frac{\epsilon}{\sigstarr}))$-strict saddle.
\end{lemma}

To deal with the case $\S$ not fixed (but as minimizer of Eq.\eqref{eq:rpca-symmetric}), we let $\U^\dag (\U^\dag)^\top$ be the best rank $r$-approximation of $\M^\star+\S^\star -\S$. The next lemma shows when $\U$ is close to $\U^\dag$ up to some rotation, $\U$ will actually be already close to $\U^\star$ up to some rotation.

\begin{lemma}\label{lem:mf_to_rpca}
There is an absolute constant $c$, assume $\gamma > c$, and $\gamma\alpha\cdot\mu r\cdot (\cn^\star)^5 \le \frac{1}{c}$. Let $\U^\dag (\U^\dag)^\top$ be the best rank $r$-approximation of $\M^\star+\S^\star -\S$, where $\S$ is the minimizer as in Eq.\eqref{eq:rpca-symmetric}. Assume
$\min_{\mR^\top \mR = \mR\mR^\top = \I}\fnorm{\U-\U^\dagger\mR} \le \epsilon$.
Let $\Delta$ be defined as in Definition \ref{def:delta}, then $\fnorm{\Delta} \le O(\epsilon\sqrt{\cn^\star})$ for polynomially small $\epsilon$.
\end{lemma}
The proof of Lemma \ref{lem:mf_to_rpca} is inspired by \citet{yi2016fast} and uses the property of the optimally chosen sparse set $\S$.
Combining these two lemmas we get our main result:

\begin{theorem}\label{thm:rpca-symmetric-main}
There is an absolute constant $c$, if $\gamma > c$, and $\gamma\alpha\cdot\mu r\cdot (\cn^\star)^5 \le \frac{1}{c}$ holds, for objective function Eq.\eqref{eq:rpca-symmetric} we have 1) all local minima satisfies $\U\U\trans = \M^\star$; 2) objective function is 
$(\epsilon, \Omega(\sigstarr), O(\frac{\epsilon \sqrt{\cn^\star}}{\sigstarr}))$-pseudo strict saddle for polynomially small $\epsilon$.
\end{theorem}

\section{Handling Asymmetric Matrices}
\label{sec:asymmetric}
In this section we show how to reduce problems on asymmetric matrices to problems on symmetric PSD matrices. 

Let $\M^\star = \U^\star \V^\star {}\trans$, and $\M = \U\V\trans$, and objective function:
\begin{align}
f(\U, \V) = 2(\M - \M^\star):\H_0:(\M - \M^\star) \nonumber + \frac{1}{2}\norm{\U\trans \U - \V\trans \V}_F^2 + Q_0(\U, \V) 
\end{align}
Note this is a scaled version of objectives introduced in Sec.\ref{sec:problems} (multiplied by $4$), and scaling will not change the property of local minima, global minima and saddle points.

We view the problem as if it is trying to find a $(d_1+d_2)\times r$ matrix, whose first $d_1$ rows are equal to $\U$, and last $d_2$ rows are equal to $\V$. 
\begin{definition}
\label{def:asymmetricquantities}
Suppose $\M^\star$ is the optimal solution, and its SVD is $\X^\star\D^\star \Y^\star{}\trans$. Let $\U^\star = \X^\star (\D^\star)^{\frac{1}{2}}$, $\V^\star = \Y^\star (\D^\star)^{\frac{1}{2}}$, $\M = \U\V^\top$ is the current point, we reduce the problem into a symmetric case using following notations.
\begin{equation}
\W = \begin{pmatrix}\U \\ \V\end{pmatrix}, 
\W^\star = \begin{pmatrix}\U^\star \\ \V^\star\end{pmatrix}, 
\N = \W\W\trans, 
\N^\star = \W^\star \W^\star {}\trans\label{eq:defwn}
\end{equation}
Further, $\Delta$ is defined to be the difference between $\W$ and $\W^\star$ up to rotation as in Definition~\ref{def:delta}.
\end{definition}
We will also transform the Hessian operators to operate on $(d_1+d_2)\times r$ matrices. In particular, define Hessian $\H_1, \G$ such that for all $\W$ we have:
\begin{align*}
 \N:\H_1:\N &= \M:\H_0:\M \\
 \N:\G:\N  &= \norm{\U\trans \U - \V\trans \V}_F^2
\end{align*}
Now, let $Q(\W) = Q(\U,\V)$, and we can rewrite the objective function $f(\W)$ as
\begin{equation}
\frac{1}{2}\left[(\N - \N^\star):4\H_1:(\N - \N^\star) + \N:\G:\N \right]+ Q(\W) \label{eq:a}
\end{equation}

We know $\H_0$ perserves the norm of low rank matrices $\M$. To reduce asymmetric problems to symmetric problem, intuitively, we also hope $\H_0$ to 
approximately preserve the norm of $\N$. However this is impossible as by definition, $\H_0$ only acts on $\M$, which is the {\em off-diagonal} blocks of $\N$. We can expect $\N:\H_0:\N$ to be close to the norm of $\U\V^\top$, but for all matrices $\U,\V$ with the same $\U\V^\top$, the matrix $\N$ can have very different norms. The easiest example is to consider $\U = \mbox{diag}(1/\epsilon,\epsilon)$ and $\V = \mbox{diag}(\epsilon,1/\epsilon)$: while $\U\V^\top = \I$ no matter what $\epsilon$ is, the norm of $\N$ is of order $1/\epsilon^2$ and can change drastically. The regularizer is exactly there to handle this case: the Hessian $\G$ of the regularizer will be related to the norm of the diagonal components, therefore allowing the full Hessian $\H = 4\H_1+\G$ to still be approximately identity. 

Now we can formalize the reduction as the following main Lemma:

\begin{lemma}\label{lem:asymmetricmain}
For the objective \eqref{eq:a}, let $\Delta, \N, \N^\star$ be defined as in Definition~\ref{def:asymmetricquantities}. Then, for any $\W\in \R^{(d_1 + d_2)\times r}$, we have
\begin{align}
\Delta : \hess f(\W) :\Delta 
\le &\Delta\Delta\trans: \H :\Delta\Delta\trans - 3(\N - \N^\star):\H:(\N - \N^\star)  \nn \\
&+ 4\la \grad f(\W), \Delta\ra + [\Delta : \hess Q(\W) : \Delta - 4 \la \grad Q(\W), \Delta \ra] \label{eq:main_asym}
\end{align}
where $\H = 4\H_1+\G$. Further, if $\H_0$ satisfies $\M:\H_0:\M \in (1\pm \delta)\|\M\|_F^2$ for some matrix $\M = \U\V^\top$, let $\W$ and $\N$ be defined as in \eqref{eq:defwn}, then $\N:\H:\N \in (1\pm 2\delta)\|\N\|_F^2$.
\end{lemma}

Intuitively, this lemma shows the same direction of improvement works as before, and the regularizer is exactly what it requires to maintain the norm-preserving property of the Hessian.

Below we prove Theorem \ref{thm:sensing-main}, which show for matrix sensing 1) all local minima satisfy $\M = \M^\star$; 2) strict saddle property is satisfied. Other proofs are deferred to appendix.

\begin{proof}[Proof of Theorem \ref{thm:sensing-main}]
In this case, $\M:\H_0:\M = \frac{1}{m}\sum_{i=1}^m \la \A_i, \M\ra^2$ and regularization $Q(\W) = 0$.
Since $\H_0$ is $(2r, 1/20)$-RIP, by Lemma \ref{lem:asymmetricmain}, we have
$\H =4\H_1+\G$ satisfying $(2r, 1/10)$-RIP.

Similar to the symmetric case, for point $\W$ with small gradient satisfying $\fnorm{\grad f(\W)} \le \epsilon$, by $(2r, 1/10)$-RIP property of $\H$ (let $\delta = 1/10$) we have
\begin{align*}
\Delta : \hess f(\U) :\Delta
=&  \Delta\Delta\trans:\H:\Delta\Delta\trans - 3(\N - \N^\star):\H:(\N - \N^\star) +4\la \grad f(\U), \Delta\ra \\
\le & (1+\delta)\fnorm{\Delta\Delta\trans}^2 - 3(1-\delta)\fnorm{\N-\N^\star}^2
+ 4\epsilon\fnorm{\Delta}\\
\le &-(1-5\delta)\fnorm{\N-\N^\star}^2 + 4\epsilon\fnorm{\Delta} 
\\
\le &-0.4\sigstarr\fnorm{\Delta}^2+ 4\epsilon\fnorm{\Delta}
\end{align*}
The second last inequality is due to Lemma \ref{lem:bound} that $\fnorm{\Delta\Delta\trans}^2 \le 2 \fnorm{\N-\N^\star}^2$, and last inequality is due to $\delta = \frac{1}{10}$ and second part of Lemma \ref{lem:bound}. This means if $\W$ is not close to $\W^\star$, that is, if $\fnorm{\Delta} \ge \frac{20\epsilon}{\sigstarr}$,
we have $\Delta : \hess f(\W) :\Delta \le -0.2\sigstarr \fnorm{\Delta}^2$. This proves 
$(\epsilon, 0.2 \sigstarr, \frac{20\epsilon}{\sigstarr})$-strict saddle property. Take $\epsilon =0$, we know all stationary points with $\fnorm{\Delta} \neq 0$ are saddle points. This means all local minima satisfy $\W\W\trans = \N^\star$, which in particular implies $\U\V\trans = \M^\star$ because $\M^\star$ is a submatrix of $\N^\star$.  
\end{proof}



\section{Runtime}
\label{sec:runtime}

In this section we give the precise statement of Corollary~\ref{cor:runtime}: the runtime of algorithms implied by the geometric properties we prove. 

In order to translate the geometric result into runtime guarantees, many algorithms require additional smoothness conditions. We say a function $f(\x)$ is $l$-smooth if for all $\x,\y$, 
$$
\|\nabla f(\y)-\nabla f(\x)\| \le l\|\y-\x\|.
$$
This is a standard assumption in optimization. In order to avoid saddle points, say a function $f(\x)$ is $\rho$-Hessian Lipschitz if for all $\x,\y$
$$
\|\hess f(\y)-\hess f(\x)\| \le \rho\|\y-\x\|.
$$
We call an optimization algorithm {\em saddle-avoiding} if the algorithm is able to find a point with small gradient and almost positive semidefinite Hessian. 

\begin{definition}\label{def:saddle-avoid}
A local search algorithm is called {\em saddle-avoiding}, if for a function $f:\R^d\to \R$ that is $l$-smooth and $\rho$-Lipschitz Hessian, given a point $\x$ such that either $\|\nabla f(\x)\| \ge \epsilon$ or $\lambda_{\min}(\nabla^2 f(\x)) \le -\sqrt{\rho \epsilon}$, can find a point $\y$ in $\mbox{poly}(1/\epsilon,d,l,\rho)$ iterations such that $f(\y) \le f(\x) - \delta$ where $\delta = \mbox{poly}(d,l,\rho,\epsilon)$. 
\end{definition}

As a immediate corollary, we know such algorithms can find a point $\x$ such that $\|\nabla f(\x)\|\le \epsilon$ and $\lambda_{\min}(\nabla^2 f(\x)) \ge -\sqrt{\rho \epsilon}$ in $\mbox{poly}(1/\epsilon,d,l,\rho)$ iterations.

Existing results show many algorithms are saddle-avoiding, including cubic regularization \citep{nesterov2006cubic}, stochastic gradient descent \citep{ge2015escaping}, trust-region algorithms \citep{sun2015nonconvex}. The most recent algorithms \citep{jin2017escape,carmon2016accelerated, agarwal2016finding} are more efficient: in particular the number of iterations only depend poly-logarithmic on dimension $d$. Now we are ready to formally state Corollary~\ref{cor:runtime}.

\begin{corollary}\label{cor:runtimeformal}
Let $R$ be the Frobenius norm of the initial points $\U_0, \V_0$, a saddle-avoiding local search algorithm can find a point $\epsilon$-close to global optimal for  matrix sensing \eqref{eq:sensing-symmetric}\eqref{eq:sensing-asymmetric}, matrix completion \eqref{eq:completion-symmetric}\eqref{eq:completion-asymmetric} in $\mbox{poly}(R, 1/\epsilon,d,\sigstarl,1/\sigstarr)$ iterations. For robust PCA \eqref{eq:rpca-symmetric}\eqref{eq:rpca-nonconvex}, alternating between a saddle-avoiding local search algorithm and computing optimal $\S\in\mathcal{S}_{\gamma\alpha}$ will find a point $\epsilon$-close to global optimal in $\mbox{poly}(R, 1/\epsilon,d,\sigstarl,1/\sigstarr)$ iterations.
\end{corollary}

This corollary states the existence of simple local search algorithms which can efficiently optimizing non-convex objectives of matrix sensing, matrix completion and robust PCA in polynomial time. The proof essentially follows from the guarantees of the saddle-avoiding algorithm and the strict-saddle properties we prove. We will sketch the proof in Section~\ref{sec:runtimesketch}.

\paragraph{Towards faster convergence} 
For many low-rank matrices problems, in the neighborhood of local minima, objective function satisfies conditions similar to strong convexity \citep{zheng2016convergence, bhojanapalli2016global} (more precisely, the $(\alpha,\beta)$-regularity condition as Assumption A3.b in \citep{jin2017escape}).
\citet{jin2017escape} showed a principle way of how to combine these strong local structures with saddle-avoiding algorithm to give global linear convergence. Therefore, it is likely that some saddle-avoiding algorithms (such as perturbed gradient descent) can achieve linear convergence for these problems.





\section{Conclusions}
\label{sec:discussion}
In this paper we give a framework that explains the recent success in understanding optimization landscape for low rank matrix problems. Our framework connects and simplifies the existing proofs, and generalizes to new settings such as asymmetric matrix completion and robust PCA. The key observation is when the Hessian operator preserves the norm of certain matrices, one can use the same directions of improvement to prove similar optimization landscape. We show the regularizer $\frac{1}{4}\|\U^\top \U - \V^\top \V\|_F^2$ is exactly what it requires to maintain this norm preserving property in the asymmetric case.
Our analysis also allows the interaction between regularizer and Hessian to handle difficult settings such as.

For low rank matrix problems, there are generalizations such as weighted matrix factorization\citep{li2016recovery} and 1-bit matrix sensing\citep{davenport20141} where the Hessian operator may behave differently as the settings we can analyze. How to characterize the optimization landscape in these settings is still an open problem.

In order to get general ways of understanding optimization landscapes for more generally, there are still many open problems. In particular, how can we decide whether two problems are similar enough to share the same optimization landscape? A minimum requirement is that the non-convex problem should have the same {\em symmetry} structure \--- the set of equivalent global optimum should be the same. In this work, we show if the problems come from convex objective functions with similar Hessian properties, then they have the same optimization landscape. We hope this serves as a first step towards general tools for understanding optimization landscape for groups of problems.


\clearpage

\bibliographystyle{plainnat} 
\bibliography{lowrank}

\appendix
\onecolumn

\section{Proofs for Symmetric  Positive Definite Problems}

In this section we provide the missing proofs for the symmetric matrix problems. First we prove Lemma~\ref{lem:bound} which connects difference in the matrix $\U$ and the difference in the matrix $\M$.

\begingroup
\def\thetheorem{\ref{lem:bound}}
\begin{lemma}
Given matrices $\U,\U^\star \in \R^{d\times r}$, let $\M = \U\U^\top$ and $\M^\star = \U^\star(\U^\star)^\top$, and let $\Delta$ be defined as in Definition~\ref{def:delta}, then we have $\|\Delta\Delta^\top\|_F^2 \le 2\|\M - \M^\star\|_F^2$, and $\sigstarr\|\Delta\|_F^2\le \frac{1}{2(\sqrt{2}-1)}\|\M-\M^\star\|_F^2$.
\end{lemma} 
\addtocounter{theorem}{-1}
\endgroup

\begin{proof}
Recall in Definition~\ref{def:delta}, $\Delta = \U - \U^\star\mR_\U$ where
\begin{equation*}
\mR_\U = \argmin_{\mR^\top \mR = \mR\mR^\top = \I}\fnorm{\U-\U^\star\mR}^2.
\end{equation*}
We first prove following claim, which will used in many places across this proof:
\begin{equation}\label{eq:claim_PSD}
\U\trans \U^\star \mR_\U \text{~is a symmetric PSD matrix.}
\end{equation}
This because by expanding the Frobenius norm, and letting the SVD of $\U^\star{}\trans\U$ be $\A\D\B\trans$, we have:
\begin{align*}
&\argmin_{\mR:\mR\mR\trans=\mR\trans\mR = \I}\fnorm{\U - \U^\star\mR}^2
= \argmin_{\mR:\mR\mR\trans=\mR\trans\mR = \I}-\la\U,  \U^\star\mR\ra \\
= &\argmin_{\mR:\mR\mR\trans=\mR\trans\mR = \I}-\tr(\U\trans\U^\star\mR)
= \argmin_{\mR:\mR\mR\trans=\mR\trans\mR = \I}-\tr(\D\A\trans\mR\B)
\end{align*}
Since $\A, \B, \mR \in \R^{r\times r}$ are all orthonormal matrix, we know $\A\trans\mR\B$ is also orthonormal matrix. Moreover for any orthonormal matrix $\T$, we have:
\begin{equation*}
\tr(\D\T) = \sum_i \D_{ii}\T_{ii} \le \sum_i \D_{ii}
\end{equation*}
The last inequality is because $\D_{ii}$ is singular value thus non-negative, and $\T$ is orthonormal, thus $\T_{ii} \le 1$. This means the maximum of $\tr(\D\T)$ is achieved when $\T = \I$, i.e., the minimum of $-\tr(\D\A\trans\mR\B)$ is achieved when $\mR = \A\B\trans$. Therefore, $\U\trans \U^\star \mR_\U = \B \D\A\trans \A\B\trans = \B\D\B\trans$ is symmetric PSD matrix.

With Eq.\eqref{eq:claim_PSD}, the remaining of proof directly follows from the results by substituting $(\U, \Y)$ in Lemma \ref{lem:aux_delta2} and \ref{lem:aux_deltalinear} with  $(\U^\star \mR_\U, \U)$.
\end{proof}

Now we are ready to prove the main lemma. 

\begingroup
\def\thetheorem{\ref{lem:main}}
\begin{lemma}[Main]
For the objective \eqref{eq:symmetricnonconvexobj}, let $\Delta$ be defined as in Definition~\ref{def:delta} and $\M = \U\U\trans$. Then, for any $\U\in \R^{d\times r}$, we have
\begin{align}
\quad\Delta : \hess f(\U) :\Delta\nonumber
=&  \Delta\Delta\trans:\H:\Delta\Delta\trans - 3(\M - \M^\star):\H:(\M - \M^\star)\nonumber\\ &+ 4\la \grad f(\U), \Delta\ra
+ [\Delta : \hess Q(\U) : \Delta - 4 \la \grad Q(\U), \Delta \ra] \nn
\end{align}
\end{lemma}
\addtocounter{theorem}{-1}
\endgroup

\begin{proof}
Recall the objective function is:
\begin{equation*}
f(\U) = \frac{1}{2}(\U\U\trans - \M^\star):\H:(\U\U\trans - \M^\star) + Q(\U)
\end{equation*}
and let $\M = \U\U\trans$. Calculating gradient and Hessian, we have for any $\mZ \in \R^{d\times r}$:
\begin{align}
\la \grad f(\U), \mZ \ra =& (\M - \M^\star):\H:(\U\mZ\trans + \mZ\U\trans) + \la \grad Q(\U), \mZ \ra  \label{eq:sym_grad}\\
\mZ : \hess f(\U) : \mZ =& (\U\mZ\trans + \mZ\U\trans):\H:(\U\mZ\trans + \mZ\U\trans) + 2(\M - \M^\star):\H:\mZ\mZ\trans + \mZ : \hess Q(\U):\mZ
\nn
\end{align}
Let $\mZ = \Delta = \U - \U^\star \mR$ as in Definition \ref{def:delta} and note $\M - \M^\star + \Delta\Delta\trans = \U\Delta\trans + \Delta\U\trans$, then
\begin{align*}
\Delta : \hess f(\U) :\Delta
= & (\U\Delta\trans + \Delta\U\trans): \H : (\U\Delta\trans + \Delta\U\trans)
+ 2(\M - \M^\star):\H:\Delta\Delta\trans  + \Delta : \hess Q(\U) :\Delta\\
= &(\M - \M^\star + \Delta\Delta\trans): \H : (\M - \M^\star + \Delta\Delta\trans)
+ 2(\M - \M^\star):\H:\Delta\Delta\trans
+ \Delta : \hess Q(\U) : \Delta\\
= & \Delta\Delta\trans:\H:\Delta\Delta\trans - 3(\M - \M^\star):\H:(\M - \M^\star) + 4 (\M - \M^\star):\H:(\M - \M^\star + \Delta\Delta\trans) \\
 &+ \Delta : \hess Q(\U) : \Delta\\
= & \Delta\Delta\trans:\H:\Delta\Delta\trans - 3(\M - \M^\star):\H:(\M - \M^\star) + 4\la \grad f(\U), \Delta\ra \\
&+ [\Delta : \hess Q(\U) : \Delta - 4 \la \grad Q(\U), \Delta \ra]
\end{align*}
where in last line, we use the calculation of gradient $\grad f(\U)$ in Eq.\ref{eq:sym_grad}. This finishes the proof.
\end{proof}

In the subsequent subsections we will prove the guarantees for matrix completion and robust PCA. The proof of matrix sensing is already given in Section~\ref{sec:symmetric}.

\subsection{Matrix Completion}

For matrix completion, the crucial component of the proof is the interaction between regularizer and the Hessian. We first state the properties (gradient and Hessian) of the regularizer $Q$ here:

\begin{lemma}\label{lem:calc_reg_mc}
The gradient and the hessian of regularization $Q(\U) = \lambda \sum_{i=1}^d (\norm{\e_i\trans \U} - \alpha)^4_{+}$ is:
\begin{align}
\la \grad Q(\U), \mZ \ra =& 4\lambda \sum_{i=1}^d(\norm{\e_i\trans \U} - \alpha)^3_{+}\frac{\e_i\trans \U\mZ\trans\e_i}{\norm{\e_i\U}} \label{eq:mc_reg1}\\
\mZ:\hess Q(\U):\mZ = & 4\lambda \sum_{i=1}^d(\norm{\e_i\trans \U} - \alpha)^3_{+} \frac{\norm{\e_i\trans \U}^2\norm{\e_i\trans\mZ}^2
- (\e_i\trans \U\mZ\trans\e_i)}{\norm{\e_i\trans \U}^3} \nn\\
&+ 12\lambda \sum_{i=1}^d(\norm{\e_i\trans \U} - \alpha)^2_{+} \left(\frac{\e_i\trans \U\mZ\trans\e_i}{\norm{\e_i\U}}\right)^2
\label{eq:mc_reg2}
\end{align}
\end{lemma}
\begin{proof}
This Lemma follows from direct calculation using linear algebra and calculus.
\end{proof}

In the first step of our framework, we hope to show that the regularizer forces the matrix $\U$ to not have large rows. This is formalized and proved below (the Lemma is similar to Lemma 4.7 in \cite{ge2016matrix}, but we get a stronger guarantee here):

\begingroup
\def\thetheorem{\ref{lem:incoherence}}
\begin{lemma}
There exists an absolute constant $c$, when sample rate $p \ge \Omega(\frac{\mu r}{d} \log d)$, $\alpha^2 = \Theta(\frac{\mu r\sigstarl}{d})$ and $\lambda = \Theta(\frac{d}{\mu r \cn^\star})$, we have for any points $\U$ with $\fnorm{\grad f (\U)} \le \epsilon$ for polynomially small $\epsilon$, with probability at least $1-1/\poly(d)$:
\begin{align*}
\max_i\norm{\e_{i}\trans \U}^2 \le O\left(\frac{(\mu r)^{1.5} \cn^\star \sigstarl}{d}\right)
\end{align*}
\end{lemma}
\addtocounter{theorem}{-1}
\endgroup

\begin{proof}
Recall the calculation of gradient:
\begin{align*}
\grad f(\U) = & \frac{2}{p}(\M -\M^\star)_\Omega \U + \grad Q(\U)
\end{align*}
where by Lemma \ref{lem:calc_reg_mc}, the gradient of regularizer is:
\begin{equation*}
\grad Q(\U) = 4\lambda \sum_{i=1}^d(\norm{\e_i\trans \U} - \alpha)^3_{+}\frac{\e_i\e_i\trans \U}{\norm{\e_i\U}^2} 
\end{equation*}

Let $i^\star = \argmax_i \norm{\e_i\trans \U}$ be the row index with maximum 2-norm.
If $\norm{\e_{i^\star}\trans \U} < 2\alpha$, by the choice of $\alpha$ in Lemma \ref{lem:incoherence} we immediately prove the lemma.
In case of $\norm{\e_{i^\star}\trans \U} \ge 2\alpha$,  consider gradient along $\e_{i^\star}\e_{i^\star}\trans \U$ direction. Since $\fnorm{\grad f(\U)} \le \epsilon$, we have
$\la \grad f(\U), \e_{i^\star}\e_{i^\star}\trans \U\ra  =\la\e_{i^\star}\trans \grad f(\U), \e_{i^\star}\trans \U\ra \le \epsilon \norm{\e_{i^\star}\trans \U}$. Therefore, with $1-1/\poly(d)$ probability, following holds: 
\begin{align*}
\epsilon \norm{\e_{i^\star}\trans \U}\ge &\la\e_{i^\star}\trans \grad f(\U), \e_{i^\star}\trans \U\ra
= \la\e_{i^\star}\trans [\frac{2}{p}(\U\U\trans -\M^\star)_\Omega \U + \grad Q(\U)], \e_{i^\star}\trans \U\ra \\
\ge &4\lambda (\norm{\e_{i^\star}\trans \U} - \alpha)^3_{+}\norm{\e_{i^\star}\trans \U}
- \frac{2}{p}\la \e_{i^\star}\trans(\M^\star)_\Omega, \e_{i^\star}\trans(\U\U\trans)_\Omega\ra \\
\ge &\frac{\lambda}{2}\norm{\e_{i^\star}\trans \U}^4 - 2\frac{1}{\sqrt{p}}\norm{\e_{i^\star}\trans(\M^\star)_\Omega} \cdot \frac{1}{\sqrt{p}}\norm{\e_{i^\star}\trans(\U\U\trans)_\Omega}\\
\ge &\frac{\lambda}{2}\norm{\e_{i^\star}\trans \U}^4 -  2\sqrt{1+0.01}\norm{\e_{i^\star}\trans\M^\star} \cdot O(\sqrt{d})\norm{\U\U\trans}_\infty \\
\ge &\frac{\lambda}{2}\norm{\e_{i^\star}\trans \U}^4 - O(\sqrt{\mu r} \sigstarl) \norm{\e_{i^\star}\trans \U}^2
\end{align*}
where second inequality use the fact $\la \e_{i^\star}\trans (\U\U\trans)_\Omega \U, \e_{i^\star}\trans \U\ra  = \norm{\e_{i^\star}\trans (\U\U\trans)_\Omega}^2 \ge 0 $; third inequality is by Cauchy-Swartz; second last inequality is by Lemma \ref{lem:T_mc} and our choice  $p \ge \Omega(\frac{\mu r}{d} \log d)$ with large enough constant, we have
$\frac{1}{\sqrt{p}}\norm{\e_{i^\star}\trans(\M^\star)_\Omega} \le  \sqrt{1+0.01} \norm{\e_{i^\star}\trans\M^\star}$, and by Lemma \ref{lem:row_mc} we have $\frac{1}{\sqrt{p}}\norm{\e_{i^\star}\trans(\U\U\trans)_\Omega}
\le O(\sqrt{d})\norm{\U\U\trans}_\infty $; the last inequality is because $\norm{\e_{i^\star}\trans\M^\star}
\le \sqrt{\frac{\mu r}{d}} \sigstarl$ as $\M^\star$ is $\mu$-incoherent.

Rearrange terms, we have:
\begin{equation*}
\norm{\e_{i^\star}\trans \U}^3 \le O(\frac{\sqrt{\mu r} \sigstarl}{\lambda})\norm{\e_{i^\star}\trans \U} + \frac{2\epsilon}{\lambda}
\end{equation*}
By choosing $\epsilon$ small enough to satisfy $(\frac{\epsilon}{\lambda})^{\frac{2}{3}} \le \frac{\sqrt{\mu r} \cdot \sigstarl }{\lambda}$, this gives:
\begin{equation*}
\max_{i}\norm{\e_{i}\trans \U}^2 \le c \cdot \max\left\{ \alpha^2, \frac{\sqrt{\mu r} \cdot \sigstarl }{\lambda}\right\}
\end{equation*}
Finally, substituting our choice of $\alpha^2$ and $\lambda$, we finished the proof.
\end{proof}

In the second step, we need to prove that the sum of Hessian $\mathcal{H}$ related terms in Equation \eqref{eq:main} is significantly negative when $\U$ and  $\U^\star$ are not close.

\begingroup
\def\thetheorem{\ref{lem:step2_mc}}
\begin{lemma}
When sample rate $p \ge \Omega(\frac{\mu^3 r^4 (\cn^\star)^4 \log d}{d})$, by choosing $\alpha^2 = \Theta(\frac{\mu r\sigstarl}{d})$ and $\lambda = \Theta(\frac{d}{\mu r \cn^\star})$ with probability at least $1-1/\poly(d)$, for all $\U$ with $\fnorm{\grad f (\U)} \le \epsilon$ for polynomially small $\epsilon$, we have
\begin{equation*}
\Delta\Delta\trans:\H:\Delta\Delta\trans - 3(\M - \M^\star):\H:(\M - \M^\star) \le -0.3 \sigstarr \fnorm{\Delta}^2
\end{equation*}
\end{lemma}
\addtocounter{theorem}{-1}
\endgroup

\begin{proof}
The key problem here is for the matrix $\Delta$, it captures the difference between $\U$ and $\U^\star$ and could have norms concentrate in very few columns. Note that when $\Delta$ is not incoherent, Hessian will still perserve norm for matrices like $\Delta\U^\top$ Lemma~\ref{lem:T_mc}, but not necessarily perserve norm for matrices like $\Delta\Delta^\top$. Therefore, we use different concentration lemmas in different regimes (divided according to whether $\Delta$ is small or large).  

First by our choice of $\alpha, \lambda$ and Lemma \ref{lem:incoherence}, we know with $1-1/\poly(d)$ probabilty, the maximum 2-norm of any row of $\U$ will be small:
\begin{align*}
\max_i \norm{\e_{i}\trans \U}^2 \le O\left(\frac{(\mu r)^{1.5} \cn^\star \sigstarl}{d}\right)
\end{align*}

\noindent\textbf{Case 1:} $\fnorm{\Delta}^2 \le \sigstarr/4$.

In this case, $\Delta$ is small, and $\Delta\Delta^\top$ is even smaller. Although $\H$ does not perserve norm for $\Delta\Delta^\top$ very well, it will only contribute a very small factor to overall summation. 
Specifically, by our choice of $p$, we will have Lemma \ref{lem:T_mc} holds with small constant $\delta$ and thus:
\begin{equation*}
\frac{1}{p}\norm{\U^\star\Delta\trans}^2_{\Omega} \ge (1 -  \delta)\fnorm{\U^\star\Delta\trans}^2 \ge (1-\delta)\sigstarr\fnorm{\Delta}^2
\end{equation*}
On the other hand,  by Lemma \ref{lem:Delta_mc}, we know:
\begin{align*}
\frac{1}{p}\norm{\Delta\Delta\trans}^2_{\Omega} \le &\fnorm{\Delta}^4 +  O(\sqrt{\frac{d}{p}} \cdot \frac{(\mu r)^{1.5} \cn^\star \sigstarl}{d})\fnorm{\Delta}^2
\le \fnorm{\Delta}^4 + \frac{\sigstarr}{4}\fnorm{\Delta}^2
\le  \frac{\sigstarr}{2}\fnorm{\Delta}^2
\end{align*}
this gives the summation:
\begin{align*}
&\Delta\Delta\trans:\H:\Delta\Delta\trans - 3(\M - \M^\star):\H:(\M - \M^\star) \\
=& \Delta\Delta\trans:\H:\Delta\Delta\trans - 3(\U^\star\Delta\trans + \Delta\U^\star{}\trans + \Delta\Delta\trans):\H:(\U^\star\Delta\trans + \Delta\U^\star{}\trans+\Delta\Delta\trans) \\
\le& -12 (\U^\star\Delta\trans:\H:\Delta\Delta\trans + \U^\star\Delta\trans:\H:\U^\star\Delta\trans)\\
\le& -\frac{12}{p} (\norm{\U^\star\Delta\trans}^2_{\Omega} - \norm{\U^\star\Delta\trans}_{\Omega}\norm{\Delta\Delta\trans}_{\Omega})
= -\frac{12}{p} \norm{\U^\star\Delta\trans}_{\Omega}(\norm{\U^\star\Delta\trans}_{\Omega} - \norm{\Delta\Delta\trans}_{\Omega}) \\
\le& -12\sqrt{1-\delta}(\sqrt{1-\delta} - \sqrt{2/3})\sigstarr\fnorm{\Delta}^2
\le-1.2\sigstarr\fnorm{\Delta}^2
\end{align*}
The last inequality is by choosing $p \ge \Omega(\frac{\mu^3 r^4 (\cn^\star)^4 \log d}{d})$ with large enough constant factor, we have small $\delta$.

~

\noindent\textbf{Case 2:} $\fnorm{\Delta}^2 \ge \sigstarr/4$.

In this case $\Delta$ is large, by Lemma \ref{lem:global_mc} with high probability, our choice of $p$ gives:
\begin{align*}
\frac{1}{p}\norm{\Delta\Delta\trans}^2_{\Omega} \le& \fnorm{\Delta\Delta\trans}^2 + O\left(\frac{dr\log d}{p}\norm{\Delta\Delta\trans}^2_{\infty}  + \sqrt{\frac{dr\log d}{p}} \fnorm{\Delta\Delta\trans}\norm{\Delta\Delta\trans}_{\infty}\right) \\
\le& \fnorm{\Delta\Delta\trans}^2 +O\left(\frac{dr\log d}{p} \cdot \frac{(\mu r)^{3} (\cn^\star \sigstarl)^2}{d^2} 
+ \sqrt{\frac{dr\log d}{p} \cdot \frac{(\mu r)^{3} (\cn^\star \sigstarl)^2}{d^2}} \fnorm{\Delta}^2 \right) \\
\le& \fnorm{\Delta\Delta\trans}^2 + \frac{(\sigstarr)^2}{80} + \frac{\sigstarr}{20}\fnorm{\Delta}^2 \le \fnorm{\Delta\Delta\trans}^2 +0.1\sigstarr\fnorm{\Delta}^2
\end{align*}
Again by Lemma \ref{lem:global_mc} with high probability
\begin{align*}
\frac{1}{p}\norm{\M - \M^\star}^2_{\Omega} \ge& \fnorm{\M-\M^\star}^2 - O\left(\frac{dr\log d}{p}\norm{\M-\M^\star}^2_{\infty}  + \sqrt{\frac{dr\log d}{p}} \fnorm{\M-\M^\star}\norm{\M-\M^\star}_{\infty}\right) \\
\ge& \fnorm{\M-\M^\star}^2 -O\left(\frac{dr\log d}{p} \cdot \frac{(\mu r)^{3} (\cn^\star \sigstarl)^2}{d^2} 
+ \sqrt{\frac{dr\log d}{p} \cdot \frac{(\mu r)^{3} (\cn^\star \sigstarl)^2}{d^2}} \fnorm{\M-\M^\star}\right) \\
\ge& \fnorm{\M-\M^\star}^2 - \frac{(\sigstarr)^2}{80} - \frac{\sigstarr}{20}\fnorm{\M-\M^\star}
\ge 0.95\fnorm{\M-\M^\star}^2 -0.1\sigstarr\fnorm{\Delta}^2
\end{align*}
This gives:
\begin{align*}
&\Delta\Delta\trans:\H:\Delta\Delta\trans - 3(\M - \M^\star):\H:(\M - \M^\star) \\
\le& \fnorm{\Delta\Delta\trans}^2 +0.1\sigstarr\fnorm{\Delta}^2 - 3(0.95\fnorm{\M-\M^\star}^2 -0.1\sigstarr\fnorm{\Delta}^2) \\
\le& -0.85\fnorm{\M-\M^\star}^2 + 0.4\sigstarr\fnorm{\Delta}^2
\le -0.3\sigstarr\fnorm{\Delta}^2
\end{align*}
where the last step is by Lemma \ref{lem:bound}. This finishes the proof.
\end{proof}

Finally, as in step 3 of our framework, we need to bound the contribution from the regularizer to Equation \eqref{eq:main}.

\begingroup
\def\thetheorem{\ref{lem:reg_mc}}
\begin{lemma}
By choosing $\alpha^2 = \Theta(\frac{\mu r\sigstarl}{d}) $ and $\lambda \alpha^2 \le  O(\sigstarr)$, we have:
\begin{equation*}
\frac{1}{4}[\Delta : \hess Q(\U) : \Delta - 4 \la \grad Q(\U), \Delta \ra] \le 0.1\sigstarr \fnorm{\Delta}^2 
\end{equation*}
\end{lemma}
\addtocounter{theorem}{-1}
\endgroup


\begin{proof}
By Lemma \ref{lem:calc_reg_mc}, the contribution from the regularizer to Equation \eqref{eq:main} can be calculated as follows:

\begin{align*}
\frac{1}{4}[\Delta : \hess Q(\U) : \Delta - 4 \la \grad Q(\U), \Delta \ra]
= &\lambda \sum_{i=1}^d(\norm{\e_i\trans \U} - \alpha)^3_{+} \frac{\norm{\e_i\trans \U}^2\norm{\e_i\trans\Delta}^2
- (\e_i\trans \U\Delta\trans\e_i)^2}{\norm{\e_i\trans \U}^3} \\
&+ 3\lambda \sum_{i=1}^d(\norm{\e_i\trans \U} - \alpha)^2_{+} \left(\frac{\e_i\trans \U\Delta\trans\e_i}{\norm{\e_i\trans\U}}\right)^2\\
&- 4\lambda \sum_{i=1}^d(\norm{\e_i\trans \U} - \alpha)^3_{+}\frac{\e_i\trans \U\Delta\trans\e_i}{\norm{\e_i\trans\U}}
\end{align*}
Denote three terms in RHS to be $A_1, A_2, A_3$. Since $\norm{\e_i\trans \U - \e_i\trans\Delta} = \norm{\e_i\trans \U^\star\mR} =\norm{\e_i\trans \U^\star}
\le \sqrt{\mu \frac{r}{d} \sigstarl}$. By choosing $\alpha > C \sqrt{\mu \frac{r}{d} \sigstarl}$ for some large constant $C$. Thus, when $\norm{\e_i\trans \U} - \alpha>0$, we have $\e_i\trans \U \approx \e_i\trans \Delta$, strictly speaking:
\begin{equation*}
\e_i\trans \U\Delta\trans\e_i  = \e_i\trans \U(\U - \U^\star \mR)\trans\e_i \ge \norm{\e_i\trans\U}^2 - \norm{\e_i\trans \U}\norm{\e_i\trans \U^\star}
\ge (1-\frac{1}{C})\norm{\e_i\trans\U}^2
\end{equation*}
and 
\begin{equation*}
\norm{\e_i\trans \U}\norm{\e_i\trans\Delta} \le \norm{\e_i\trans \U}(\norm{\e_i\trans \U} + \norm{\e_i\trans \U^\star})
\le (1+\frac{1}{C})\norm{\e_i\trans\U}^2
\end{equation*}
Now we bound the summation $A_1 + A_2 + A_3$ by seperately bounding $A_1 + 0.1A_3$ and $A_2 + 0.9A_3$. First, we have:
\begin{equation*}
A_1 + 0.1 A_3 \le \lambda \sum_{i=1}^d(\norm{\e_i\trans \U} - \alpha)^3_{+}\norm{\e_i\trans \U} \left((1+\frac{1}{C})^2 - (1-\frac{1}{C})^2-0.4(1-\frac{1}{C})\right)<0
\end{equation*}
Then, the remaining part is:
\begin{align*}
A_2 + 0.9 A_3 = 3\lambda\sum_{i=1}^d(\norm{\e_i\trans \U} - \alpha)^2_{+} \frac{\e_i\trans \U\Delta\trans\e_i}{\norm{\e_i\trans\U}}\left[\frac{\e_i\trans \U\Delta\trans\e_i}{\norm{\e_i\trans\U}} - 1.2(\norm{\e_i\trans \U} - \alpha)_{+}  \right]
\end{align*}
We further denote $i$-th summand in RHS as $A^{(i)}_2 + 0.9 A^{(i)}_3$, and 
 decompose this term as $A_2 + 0.9 A_3 = \sum_{i=1}^d (A^{(i)}_2 + 0.9 A^{(i)}_3)$.

~

\noindent Case 1: for $i$ such that $\norm{\e_i\trans\U} \ge 9\alpha$, and $C\ge100$, we have:
\begin{align*}
A^{(i)}_2 + 0.9 A^{(i)}_3 \le 3\lambda(\norm{\e_i\trans \U} - \alpha)^2_{+} \frac{\e_i\trans \U\Delta\trans\e_i}{\norm{\e_i\trans\U}}\left[(1+\frac{1}{C})\norm{\e_i\trans \U} - 1.2(\norm{\e_i\trans \U} - \alpha)_{+}  \right] \le 0
\end{align*}
This is because in the above product, we have $(\norm{\e_i\trans \U} - \alpha)^2_{+} >0$, $\frac{\e_i\trans \U\Delta\trans\e_i}{\norm{\e_i\trans\U}} \ge (1-\frac{1}{C})\norm{\e_i\trans\U}\ge 0$ and $\left[(1+\frac{1}{C})\norm{\e_i\trans \U} - 1.2(\norm{\e_i\trans \U} - \alpha)_{+}  \right] \le 0$.

~

\noindent Case 2: for $i$ such that $\alpha<\norm{\e_i\trans\U} < 9\alpha$, we call this set $I = \{i~|~\alpha<\norm{\e_i\trans\U} < 9\alpha\} $:
\begin{align*}
\sum_{i\in I} A^{(i)}_2 + 0.9 A^{(i)}_3 
\le 3\cdot 10^4 \times \lambda |I| \alpha^4 
\end{align*}

In sum, this proves there exists some large constant $c_2$ so the regularization term:
\begin{align*}
\frac{1}{4}[\Delta : \hess Q(\U) : \Delta - 4 \la \grad Q(\U), \Delta \ra] \le c_2 \lambda |I| \alpha^4
\end{align*}
Finally, by the property of set $I$:
\begin{equation*}
\sigstarr \fnorm{\Delta}^2 \ge \sigstarr \sum_{i\in I} \norm{\e_i\trans \Delta}^2 \ge \sigstarr |I| \alpha^2
\end{equation*}
Therefore, as long as $\lambda \alpha^2 \le  \sigstarr/c_3$ for some large absolute constant $c_3$ (which is satisfied by our choice of $\lambda$), we have
\begin{equation*}
[\Delta : \hess Q(\U) : \Delta - 4 \la \grad Q(\U), \Delta \ra]  \le 0.1\sigstarr \fnorm{\Delta}^2 
\end{equation*}
\end{proof}

Combining these lemmas, we are now ready to prove the main theorem for symmetric matrix completion.

\begingroup
\def\thetheorem{\ref{thm:main_mc}}
\begin{theorem}
When sample rate $p \ge \Omega(\frac{\mu^3 r^4 (\cn^\star)^4 \log d}{d})$, by choosing $\alpha^2 = \Theta(\frac{\mu r\sigstarl}{d})$ and $\lambda = \Theta(\frac{d}{\mu r \cn^\star})$. 
Then with probability at least $1-1/\poly(d)$, for matrix completion objective \eqref{eq:completion-symmetric} we have 1) all local minima satisfy $\U\U^\top = \M^\star$ 2) the function is $(\epsilon, \Omega(\sigstarr),O( \frac{\epsilon}{\sigstarr}))$-strict saddle
for polynomially small $\epsilon$. 
\end{theorem}
\addtocounter{theorem}{-1}
\endgroup

\begin{proof}
By Lemma \ref{lem:step2_mc}, we know
\begin{equation*}
\Delta\Delta\trans:\H:\Delta\Delta\trans - 3(\M - \M^\star):\H:(\M - \M^\star) \le -0.3 \sigstarr \fnorm{\Delta}^2
\end{equation*}

On the other hand, by Lemma \ref{lem:reg_mc}, we have the regularization term:
\begin{equation*}
[\Delta : \hess Q(\U) : \Delta - 4 \la \grad Q(\U), \Delta \ra]  \le 0.1\sigstarr \fnorm{\Delta}^2 
\end{equation*}

This means for point $\U$ with small gradient satisfying $\fnorm{\grad f(\U)} \le \epsilon$:
\begin{equation*}
\Delta : \hess f(\U) :\Delta \le -0.2\sigstarr \fnorm{\Delta}^2 + 4\epsilon \fnorm{\Delta}
\end{equation*}
That is, if $\U$ is close to $\U^\star$ (i.e. if $\fnorm{\Delta} \ge \frac{40\epsilon}{\sigstarr}$),
we have $\Delta : \hess f(\U) :\Delta \le -0.1\sigstarr \fnorm{\Delta}^2$. This proves 
$(\epsilon, 0.1\sigstarr, \frac{40\epsilon}{\sigstarr})$-strict saddle property. Take $\epsilon =0$, we know all stationary points with $\fnorm{\Delta} \neq 0$ are saddle points. This means all local minima are global minima (satisfying $\U\U\trans = \M^\star$), which finishes the proof.
\end{proof}

\subsection{Robust PCA}


For robust PCA, the first crucial step is to analyze the matrix factorization problem when target matrix is not necessarily low rank (that happens if we fix $\S$).

\begingroup
\def\thetheorem{\ref{lem:mf_strictsaddle}}
\begin{lemma}
Let $\A \in \R^{d\times d}$ be a symmetric PSD matrix, and matrix factorization objective to be:
\begin{equation*}
f(\U) = \fnorm{\U\U\trans - \A}^2
\end{equation*}
where $\sigma_r(\A)\ge 15\sigma_{r+1}(\A)$.
then 1) all local minima satisfies $\U\U\trans = \proj_r(\A)$ (best rank-$r$ approximation),
2) objective is $(\epsilon, \Omega(\sigstarr), O(\frac{\epsilon}{\sigstarr}))$-strict saddle.
\end{lemma}
\addtocounter{theorem}{-1}
\endgroup

\begin{proof}
Denote $\M^\star = \proj_r(\A)$ to be the top $r$ part and $\S = \A - \M^\star$ to be the remaining part.
In our framework, we can also view this remaining part as regularization term. That is:
\begin{equation*}
\M:\H:\M = \fnorm{\M}^2 \quad \text{and} \quad Q(\U) = \la \M^\star-\U\U\trans, \S\ra + \frac{1}{2}\fnorm{\S}^2
\end{equation*}
Moreover, since the eigenspace of $\M^\star$ is perpendicular to $\S$, we have:
\begin{align*}
[\Delta : \hess Q(\U) : \Delta - 4 \la \grad Q(\U), \Delta \ra] = -2 \la\Delta\Delta\trans, \S\ra + 4\la \U\Delta + \Delta\U, \S\ra
= 6\la \U\U\trans, \S\ra \le 6\norm{\S}\fnorm{\Delta}^2
\end{align*}
The last step is because suppose $\X\D\X\trans$ is the SVD of $\S$, then $\la \U\U\trans, \S\ra \le \norm{\D}\fnorm{\X\trans \U}^2
= \norm{\D}\fnorm{\X\trans (\U-\U^\star)}^2 \le \norm{\S}\fnorm{\Delta}^2$.

Therefore, for point $\U$ with small gradient satisfying $\fnorm{\grad f(\U)} \le \epsilon$:
\begin{align*}
\Delta : \hess f(\U) :\Delta
=&  \Delta\Delta\trans:\H:\Delta\Delta\trans - 3(\M - \M^\star):\H:(\M - \M^\star) +4\la \grad f(\U), \Delta\ra \\
&+ [\Delta : \hess Q(\U) : \Delta - 4 \la \grad Q(\U), \Delta \ra]\\
\le & \fnorm{\Delta\Delta\trans}^2 - 3\fnorm{\M-\M^\star}^2
+ 4\epsilon\fnorm{\Delta} + 6\norm{\S}\fnorm{\Delta}^2\\
\le &-\fnorm{\M-\M^\star}^2 + 4\epsilon\fnorm{\Delta} + 6\norm{\S}\fnorm{\Delta}^2
\le -0.4\sigstarr\fnorm{\Delta}^2+ 4\epsilon\fnorm{\Delta}
\end{align*}
The second last inequality is due to Lemma \ref{lem:aux_delta2} that $\fnorm{\Delta\Delta\trans}^2 \le 2 \fnorm{\M-\M^\star}^2$, and last inequality is due to Lemma \ref{lem:aux_deltalinear} and $\norm{\S} = \lambda_{r+1}(\A) \le \lambda_{r}(\A)/15 $. This means if $\U$ is close to $\U^\star$, that is, if $\fnorm{\Delta} \ge \frac{20\epsilon}{\sigstarr}$,
we have $\Delta : \hess f(\U) :\Delta \le -0.2\sigstarr \fnorm{\Delta}^2$. This proves 
$(\epsilon, 0.2 \sigstarr, \frac{20\epsilon}{\sigstarr})$-strict saddle property. Take $\epsilon =0$, we know all stationary points with $\fnorm{\Delta} \neq 0$ are saddle points. This means all local minima are global minima (satisfying $\U\U\trans = \M^\star$), which finishes the proof.
\end{proof}

Next we need to show that if $\U$ is close to the best rank-$r$ approximation of $\M^\star+\S^\star - \S$, then it also must be close to the true $\U^\star$. The proofs of this lemma for symmetric robust PCA is almost directly followed by the arguments for asymmetric versions. Therefore we do not repeat the proofs here.

\begingroup
\def\thetheorem{\ref{lem:mf_to_rpca}}
\begin{lemma}
There is an absolute constant $c$, assume $\gamma > c$, and $\gamma\alpha\cdot\mu r\cdot (\cn^\star)^5 \le \frac{1}{c}$. Let $\U^\dag (\U^\dag)^\top$ be the best rank $r$-approximation of $\M^\star+\S^\star -\S$, where $\S$ is the minimizer as in Eq.\eqref{eq:rpca-symmetric}. Assume
$\min_{\mR^\top \mR = \mR\mR^\top = \I}\fnorm{\U-\U^\dagger\mR} \le \epsilon$.
Let $\Delta$ be defined as in Definition \ref{def:delta}, then $\fnorm{\Delta} \le O(\epsilon\sqrt{\cn^\star})$ for polynomially small $\epsilon$.
\end{lemma}
\addtocounter{theorem}{-1}
\endgroup

\begin{proof}
The proof follows from the same argument as the proof of Lemma \ref{lem:mf_to_rpca_asym}.
\end{proof}

Combining these two lemmas, it is not hard to show to main result for Robust PCA.

\begingroup
\def\thetheorem{\ref{thm:rpca-symmetric-main}}
\begin{theorem}
There is an absolute constant $c$, if $\gamma > c$, and $\gamma\alpha\cdot\mu r\cdot (\cn^\star)^5 \le \frac{1}{c}$ holds, for objective function Eq.\eqref{eq:rpca-symmetric} we have 1) all local minima satisfies $\U\U\trans = \M^\star$; 2) objective function is 
$(\epsilon, \Omega(\sigstarr), O(\frac{\epsilon \sqrt{\cn^\star}}{\sigstarr}))$-pseudo strict saddle for polynomially small $\epsilon$.
\end{theorem}
\addtocounter{theorem}{-1}
\endgroup

\begin{proof}
Recall objective function:
\begin{equation*}
f(\U) = \frac{1}{2}\min_{\S\in \mathcal{S}_{\gamma\alpha}}\norm{\U\U\trans + \S - \M^\star - \S^\star}_F^2 
\end{equation*}
Consider point $\U$ with small gradient satisfying $\fnorm{\grad f(\U)} \le \epsilon$.
Let 
\begin{equation*}
\S_\U = \argmin_{\S\in \mathcal{S}_{\gamma\alpha}}\norm{\U\U\trans + \S - \M^\star - \S^\star}_F^2 
\end{equation*}
and function 
$f_{\U}(\tilde{\U}) = \norm{\tilde{\U}\tilde{\U}\trans + \S_{\U} - \M^\star - \S^\star}_F^2$, then, we know for all $\tilde{\U}$,
we have $f_{\U}(\tilde{\U}) \ge f(\tilde{\U})$ and $f_{\U}(\U) = f(\U)$. 
Since $f_{\U}(\tilde{\U})$ is matrix factorization objective where by Lemma \ref{lem:spectral_infty_rpca_asym}:
\begin{align*}
&\norm{ \S^\star - \S_{\U}} \le 2 \gamma \alpha \cdot 2\frac{\mu r \sigstarl}{d}  \le 0.01\sigstarr \\
&\sigma_{r}(\M^\star + \S^\star - \S_{\U}) \ge \sigstarr - \norm{ \S^\star - \S_{\U}} \ge 0.99 \sigstarr\\
&\sigma_{r+1}(\M^\star + \S^\star - \S_{\U}) \le \norm{ \S^\star - \S_{\U}} \le 0.01 \sigstarr
\end{align*}
This gives $\sigma_{r}(\M^\star + \S^\star - \S_{\U}) \ge 15 \sigma_{r+1}(\M^\star + \S^\star - \S_{\U})$. 
Given $\fnorm{\grad f_{\U}(\U)} = \fnorm{\grad f(\U)} \le \epsilon$, by Lemma \ref{lem:mf_strictsaddle_asym}, we know either $\lambda_{\min}(\hess f_{\U}(\U)) \le -0.2 \sigstarr$ or $\min_{\mR^\top \mR = \mR\mR^\top = \I}\fnorm{\U-\U^\dagger\mR} \le \frac{20\epsilon}{\sigstarr}$ where  $\U^\dagger
(\U^\dagger)\trans $ is the best rank $r$-approximation of $\M^\star+\S^\star -\S_\U$.  
By Lemma \ref{lem:mf_to_rpca_asym}, we immediately have $\fnorm{\Delta}  \le \frac{10^3\epsilon\sqrt{\cn^\star}}{\sigstarr}$, which proves 
$(\epsilon, \Omega(\sigstarr), O(\frac{\epsilon \sqrt{\cn^\star}}{\sigstarr}))$-pseudo strict saddle. By taking $\epsilon = 0$, we proved all local minima satisfies $\U\U\trans = \M^\star$.
\end{proof}


\section{Proofs for Asymmetric Problems}

In this section we give proofs for the asymmetric settings. In particular, we first prove the main lemma, which gives the crucial reduction from asymmetric case to symmetric case.

\begingroup
\def\thetheorem{\ref{lem:asymmetricmain}}
\begin{lemma}
For the objective \eqref{eq:a}, let $\Delta, \N$ be defined as in Definition~\ref{def:asymmetricquantities}. Then, for any $\W\in \R^{(d_1 + d_2)\times r}$, we have
\begin{align*}
\Delta : \hess f(\W) :\Delta
\le &\Delta\Delta\trans: \H :\Delta\Delta\trans - 3(\N - \N^\star):\H:(\N - \N^\star)  \\
&+ 4\la \grad f(\W), \Delta\ra + [\Delta : \hess Q(\W) : \Delta - 4 \la \grad Q(\W), \Delta \ra]
\end{align*}
Further, if $\H_0$ satisfies $\M:\H_0:\M \in (1\pm \delta)\|\M\|_F^2$ for some matrix $\M = \U\V^\top$, let $\W$ and $\N$ be defined as in \eqref{eq:defwn}, then $\N:\H:\N \in (1\pm 2\delta)\|\N\|_F^2$.
\end{lemma}
\addtocounter{theorem}{-1}
\endgroup

\begin{proof}
Recall the objective function is ($\N = \W\W\trans$):
\begin{equation*}
f(\W) = \frac{1}{2}\left[(\N - \N^\star):4\H_1:(\N - \N^\star) + \N:\G:\N \right]+ Q(\W)
\end{equation*}
Calculating gradient and Hessian, we have for any $\mZ \in \R^{d\times r}$:
\begin{align*}
\la \grad f(\W), \mZ \ra =& (\N - \N^\star):4\H_1:(\W\mZ\trans + \mZ\W\trans) + \N:\G:(\W\mZ\trans + \mZ\W\trans) +  \la \grad Q(\W), \mZ \ra \\
\mZ : \hess f(\W) : \mZ =& (\W\mZ\trans + \mZ\W\trans):(4\H_1 + \G):(\W\mZ\trans + \mZ\W\trans) + 2(\N - \N^\star):4\H_1:\mZ\mZ\trans \\
&+ 2\N:\G:\Delta\Delta\trans + \mZ : \hess Q(\W):\mZ
\end{align*}
Let $\mZ = \Delta = \W - \W^\star \mR$ as in Definition \ref{def:asymmetricquantities}, and note $\N - \N^\star + \Delta\Delta\trans = \W\Delta\trans + \Delta\W\trans$ and $\N^\star:\G:\N^\star = \N^\star:\G:\W^\star\W\trans = 0$ due to $\U^\star{}\trans \U^\star = \V^\star{}\trans \V^\star$. Let $\H = 4\H_1 + \G$, then
\begin{align}
\langle\grad f(\W), \Delta\rangle
 =& (\N - \N^\star):\H:(\W\Delta\trans + \Delta\W\trans)
 + \N^\star:\G:(\W\Delta\trans + \Delta\W\trans)  + \la \grad Q(\W), \Delta \ra\nn \\
 =& (\N - \N^\star):\H:(\W\Delta\trans + \Delta\W\trans)
 + 2\N^\star:\G:\N + \la \grad Q(\W), \Delta \ra \label{eq:cond_1_asym}
\end{align}
Where the last equality is use the fact $\N^\star:\G:\W^\star\W\trans = \N^\star:\G:\W\W^\star{}\trans  = 0$.
For Hessian along $\Delta$ direction: 
\begin{align}
\Delta : \hess f(\W) :\Delta
= & (\W\Delta\trans + \Delta\W\trans): \H : (\W\Delta\trans + \Delta\W\trans)
+ 2(\N - \N^\star):4\H_1:\Delta\Delta\trans \nn\\
& + 2\N:\G:\Delta\Delta\trans + \Delta : \hess Q(\W):\Delta \label{eq:asym_hessian}
\end{align}
For first term of Eq.\eqref{eq:asym_hessian}: since $\Delta\Delta\trans + (\N - \N^\star) = \W\Delta\trans + \Delta\W\trans$ and $(a+b)^2 = a^2 + 2b(a+b) - b^2$ and Eq.\eqref{eq:cond_1_asym}, we have:
\begin{align*}
&(\W\Delta\trans + \Delta\W\trans): \H : (\W\Delta\trans + \Delta\W\trans) \\
=&\Delta\Delta\trans: \H :\Delta\Delta\trans + 2 (\N - \N^\star):\H:(\W\Delta\trans + \Delta\W\trans) - (\N - \N^\star):\H:(\N - \N^\star) \\
=&\Delta\Delta\trans: \H :\Delta\Delta\trans - (\N - \N^\star):\H:(\N - \N^\star) 
+ 2\langle\grad f(\W), \Delta\rangle- 4\N^\star:\G:\N - 2\la \grad Q(\W), \Delta \ra
\end{align*}
For the sum of second and third terms of Eq.\eqref{eq:asym_hessian}:
\begin{align*}
&2(\N - \N^\star):4\H_1:\Delta\Delta\trans + 2\N:\G:\Delta\Delta\trans \\
= &2(\N - \N^\star):\H :\Delta\Delta\trans + 2\N^\star:\G:\N \\
= & - 2(\N - \N^\star):\H:(\N - \N^\star) + 2 (\N - \N^\star):\H:(\W\Delta\trans + \Delta\W\trans) + 2\N^\star:\G:\N \\
= & - 2(\N - \N^\star):\H:(\N - \N^\star) + 2\langle\grad f(\W), \Delta\rangle- 2\N^\star:\G:\N - 2\la \grad Q(\W), \Delta \ra
\end{align*}
In sum, we have:
\begin{align*}
\Delta : \hess f(\W) :\Delta 
= &\Delta\Delta\trans: \H :\Delta\Delta\trans - 3(\N - \N^\star):\H:(\N - \N^\star) - 6\N^\star:\G:\N \\
&+ 4\la \grad f(\W), \Delta\ra + [\Delta : \hess Q(\W) : \Delta - 4 \la \grad Q(\W), \Delta \ra]
\end{align*}
cleary the third term is always non-positive.
this gives:
\begin{align*}
\Delta : \hess f(\W) :\Delta 
\le &\Delta\Delta\trans: \H :\Delta\Delta\trans - 3(\N - \N^\star):\H:(\N - \N^\star) + 4\la \grad f(\W), \Delta\ra \nn \\
 &+ [\Delta : \hess Q(\W) : \Delta - 4 \la \grad Q(\W), \Delta \ra]
\end{align*}

The remaining claims directly follows from Lemma \ref{lem:sym_to_asym}.
\end{proof}


\begin{lemma}\label{lem:sym_to_asym}
Let  $\A = \begin{pmatrix}\A_0 & \A_c \\ \A_c\trans & \A_1  \end{pmatrix} \in \R^{(d_1+d_2)\times(d_1+d_2)}$ be a symmetric matrix, if for $\H_0$ we have:
\begin{equation*}
(1-\delta)\fnorm{\A_c}^2\le\A_c:\H_0:\A_c\le(1+\delta)\fnorm{\A_c}^2
\end{equation*}
Then, we have:
\begin{equation*}
(1-2\delta)\fnorm{\A}^2 \le \A:\H: \A \le (1+2\delta)\fnorm{\A}^2
\end{equation*}
\end{lemma}
\begin{proof}
By calculation,
\begin{equation*}
\A:\H: \A =  4\A_c:\H_0:\A_c + \left(\norm{\A_1}_F^2 + \norm{\A_2}_F^2 - 2\norm{\A_c}_F^2) \right)
\end{equation*}
The lemma easily follows.
\end{proof}

In the remainder of this section we prove the main theorems for matrix completion and robust PCA.

\subsection{Matrix Completion}
Across this section, we denote
\begin{equation*}
Q_1(\U) = \lambda_1 \sum_{i=1}^{d_1} (\norm{\e_i\trans \U} - \alpha_1)^4_{+} 
\quad \text{and} \quad
Q_2(\V) = \lambda_2 \sum_{j=1}^{d_2} (\norm{\e_j\trans \V} - \alpha_2)^4_{+} 
\end{equation*}
and clearly $Q(\W) = Q_1(\U) + Q_2(\V)$. We always denote $d=\max\{d_1, d_2\}$

We proceed in three steps analogous to the symmetric setting. First we show the regularizer again implies rows of $\U,\V$ cannot be too large (similar to Lemma~\ref{lem:incoherence}).

\begin{lemma} \label{lem:incoherence_asym}
Let $d = \max\{d_1, d_2\}$, there is an absolute constant $c$, when sample rate $p \ge \Omega(\frac{\mu r \log d}{\min\{d_1, d_2\}} )$, and $\alpha_1^2 = \Theta(\frac{\mu r\sigstarl}{d_1}), \alpha_2^2 = \Theta(\frac{\mu r\sigstarl}{d_2})$, $\lambda_1 = \Theta(\frac{d_1}{\mu r \cn^\star}), \lambda_2 = \Theta(\frac{d_2}{\mu r \cn^\star})$, we have for any points $\W$ with $\fnorm{\grad f (\W)} \le \epsilon$ with polynomially small $\epsilon$. with probability at least $1-1/\poly(d)$:
\begin{align*}
\max_i\norm{\e_{i}\trans \U}^2 \le O\left(\frac{\mu^2 r^{2.5} (\cn^\star)^2 \sigstarl}{d_1}\right)
\quad \text{and} \quad 
\max_j\norm{\e_{j}\trans \V}^2 \le O\left(\frac{\mu^2 r^{2.5} (\cn^\star)^2 \sigstarl}{d_2}\right)
\end{align*}
\end{lemma}
\begin{proof}
In this proof, by symmetry, W.L.O.G, we can assume
$\sqrt{d_1}\max_{i}\norm{\e_{i}\trans \U} \ge \sqrt{d_2}\max_{j}\norm{\e_{j}\trans \V} $.
We know gradient can be calculated as:
\begin{align*}
\grad f(\W) = & \frac{4}{p}
\begin{pmatrix}
(\M -\M^\star)_\Omega \V \\
(\M -\M^\star)_{\Omega}\trans \U
\end{pmatrix}
+
\begin{pmatrix}
 \U(\U\trans\U -\V\trans\V) \\
 \V(\V\trans\V - \U\trans\U)
\end{pmatrix}
+
 \grad Q(\W)
\end{align*}
Where:
\begin{equation*}
\grad Q(\W) = 4\lambda_1 \sum_{i=1}^{d_1}(\norm{\e_i\trans \W} - \alpha_1)^3_{+}\frac{\e_i\e_i\trans \W}{\norm{\e_i\W}^2} 
+ 4\lambda_2 \sum_{i=d_1+1}^{d_2}(\norm{\e_i\trans \W} - \alpha_2)^3_{+}\frac{\e_i\e_i\trans \W}{\norm{\e_i\W}^2} 
\end{equation*}
Clearly, we have $\la\grad Q(\W), \W \ra \ge 0$, therefore, for any points $\W$ with small gradient $\fnorm{\grad f (\W)} \le \epsilon$, we have:
\begin{align*}
\epsilon\fnorm{\W} \ge& \la \grad f(\W), \W\ra \\
=& \fnorm{\U\trans\U - \V\trans\V}^2 + \frac{4}{p} \la (\M -\M^\star)_\Omega, \M \ra + \la\grad Q(\W), \W \ra \\
\ge &\fnorm{\U\trans\U - \V\trans\V}^2 - \frac{4}{p}\la (\M^\star)_\Omega, (\M)_\Omega \ra \\
\ge &\fnorm{\U\trans\U - \V\trans\V}^2 - 4\cdot \frac{1}{\sqrt{p}}\norm{\M^\star}_{\Omega} \cdot \frac{1}{\sqrt{p}}\norm{\M}_{\Omega} \\
\ge &\fnorm{\U\trans\U - \V\trans\V}^2 - O(\sqrt{d_1 d_2}) \fnorm{\M^\star}\norm{\M}_{\infty}
\end{align*}
where last inequality is by Lemma \ref{lem:T_mc} and Lemma \ref{lem:row_mc}.
Let $i^\star = \argmax_i \norm{\e_i\trans \U}$, and $j^\star = \argmax_j \norm{\e_j\trans \V}$.
By assumption, we know $\sqrt{d_1}\norm{\e_{i^\star}\trans \U} \ge \sqrt{d_2}\norm{\e_{j^\star}\trans \V}$
and due to $\fnorm{\M^\star} \le \sqrt{r}\sigstarl$ and $\norm{\M}_\infty \le \norm{\e_{i^\star}\trans \U}\norm{\e_{j^\star}\trans \V}$, this gives:
\begin{equation}\label{eq:reg_F_asym_mc}
\fnorm{\U\trans\U - \V\trans\V}^2 \le O(d_1\sigstarl\sqrt{r} ) \norm{\e_{i^\star}\trans \U}^2
+ O(\epsilon d)\norm{\e_{i^\star}\trans \U}
\end{equation}

In case $\norm{\e_{i^\star}\trans \U} \ge 2\alpha_i$,  consider $\la\e_{i^\star}\trans \grad f(\W), \e_{i^\star}\trans \W\ra$:
\begin{align*}
\epsilon \norm{\e_{i^\star}\trans \U}\le &\la\e_{i^\star}\trans \grad f(\W), \e_{i^\star}\trans \W\ra\\
= &\la\e_{i^\star}\trans \left[\frac{4}{p}
(\M -\M^\star)_\Omega \V 
+
 \U(\U\trans\U -\V\trans\V) 
+ \grad Q_1(\U)\right], \e_{i^\star}\trans \U\ra \\
\ge &4\lambda_1 (\norm{\e_{i^\star}\trans \U} - \alpha_1)^3_{+}\norm{\e_{i^\star}\trans \U}
- \frac{4}{p}\la \e_{i^\star}\trans(\M^\star)_\Omega, \e_{i^\star}\trans(\M)_\Omega\ra 
- \fnorm{\U\trans\U - \V\trans\V}\norm{\e_{i^\star}\trans \U}^2\\
\ge &\frac{\lambda_1}{2}\norm{\e_{i^\star}\trans \U}^4 - 4\frac{1}{\sqrt{p}}\norm{\e_{i^\star}\trans(\M^\star)_\Omega} \cdot \frac{1}{\sqrt{p}}\norm{\e_{i^\star}\trans(\M)_\Omega} -\fnorm{\U\trans\U - \V\trans\V}\norm{\e_{i^\star}\trans \U}^2\\
\ge &\frac{\lambda_1}{2}\norm{\e_{i^\star}\trans \U}^4 -  O(1)\norm{\e_{i^\star}\trans\M^\star} \cdot \sqrt{d_2}\norm{\M}_\infty - \fnorm{\U\trans\U - \V\trans\V}\norm{\e_{i^\star}\trans \U}^2\\
\ge &\frac{\lambda_1}{2}\norm{\e_{i^\star}\trans \U}^4 - \sqrt{\mu r} \sigstarl \norm{\e_{i^\star}\trans \U}^2 - \fnorm{\U\trans\U - \V\trans\V}\norm{\e_{i^\star}\trans \U}^2
\end{align*}
where second last inequality is by Lemma \ref{lem:T_mc} and Lemma \ref{lem:row_mc}. Substitute in Eq.\eqref{eq:reg_F_asym_mc}, we have:
\begin{equation*}
\lambda_1 \norm{\e_{i^\star}\trans \U}^3 \le
O(\sqrt{\mu r} \sigstarl) \norm{\e_{i^\star}\trans \U}
+ O(\sqrt{d\sigstarl}\cdot r^{\frac{1}{4}} ) \norm{\e_{i^\star}\trans \U}^2
+ \epsilon + O(\sqrt{\epsilon d}) \norm{\e_{i^\star}\trans \U}^{1.5}
\end{equation*}
by choosing $\epsilon$ to be polynomially small, we have:
\begin{equation*}
\sqrt{\frac{d_2}{d_1}}\max_{j}\norm{\e_{j}\trans \V} \le \max_{i}\norm{\e_{i}\trans \U}^2 \le c \max\left\{\alpha_1^2, \frac{\sqrt{\mu r} \cdot \sigstarl }{\lambda_1}, \frac{d_1\sigstarl\sqrt{r}}{\lambda_1^2}\right\}
\end{equation*}
Finally, substituting our choice of $\alpha^2$ and $\lambda$, we finished the proof.
\end{proof}

Next we show the Hessian $\mathcal{H}$ related terms in Eq.\eqref{eq:main_asym} is negative when $\W\ne \W^\star$. This is analogous to Lemma~\ref{lem:step2_mc}.


\begin{lemma} \label{lem:step2_mc_asym}
Let $d = \max\{d_1, d_2\}$, when sample rate $p \ge \Omega(\frac{\mu^4 r^6 (\cn^\star)^6 \log d}{\min\{d_1, d_2\}})$, by choosing $\alpha_1^2 = \Theta(\frac{\mu r\sigstarl}{d_1}), \alpha_2^2 = \Theta(\frac{\mu r\sigstarl}{d_2})$ and $\lambda_1 = \Theta(\frac{d_1}{\mu r \cn^\star}), \lambda_2 = \Theta(\frac{d_2}{\mu r \cn^\star})$. Then with probability at least $1-1/\poly(d)$, for all $\W$ with $\fnorm{\grad f (\W)} \le \epsilon$ for polynomially small $\epsilon$:
\begin{equation*}
\Delta\Delta\trans:\H:\Delta\Delta\trans - 3(\M - \M^\star):\H:(\M - \M^\star) \le -0.3 \sigstarr \fnorm{\Delta}^2
\end{equation*}
\end{lemma}
\begin{proof}
Again the idea is similar, we divide into cases according to the norm of $\Delta$ and use different concentration inequalities.
By our choice of $\alpha, \lambda$ and Lemma \ref{lem:incoherence_asym}, we known when $\epsilon$ is polynomially small, 
with high probability:
\begin{align*}
\max_i\norm{\e_{i}\trans \U}^2 \le O\left(\frac{\mu^2 r^{2.5} (\cn^\star)^2 \sigstarl}{d_1}\right)
\quad \text{and} \quad 
\max_j\norm{\e_{j}\trans \V}^2 \le O\left(\frac{\mu^2 r^{2.5} (\cn^\star)^2 \sigstarl}{d_2}\right)
\end{align*}
In this proof, we denote $\Delta = (\Delta_\U\trans, \Delta_\V\trans)\trans$, clearly, we have $\fnorm{\Delta_\U} \le \fnorm{\Delta}$ and  $\fnorm{\Delta_\V} \le \fnorm{\Delta}$.

\noindent\textbf{Case 1:} $\fnorm{\Delta}^2 \le \sigstarr/40$.
By Lemma \ref{lem:T_mc} and Lemma \ref{lem:sym_to_asym}, we know:
\begin{equation*}
\W^\star\Delta\trans: \H:\W^\star\Delta\trans \ge (1 -  2\delta)\fnorm{\W^\star\Delta\trans}^2 \ge (1-2\delta)\sigstarr\fnorm{\Delta}^2
\end{equation*}
On the other hand, by Lemma \ref{lem:Delta_mc} and our choice of $p$, we have:
\begin{align*}
\frac{1}{p}\norm{\Delta_\U\Delta_\V\trans}^2_{\Omega} \le &(1+\delta)\fnorm{\Delta_\U}^2\fnorm{\Delta_\V}^2 +  O(\sqrt{\frac{d}{p}} \cdot \frac{\mu^2 r^{2.5} (\cn^\star)^2 \sigstarl}{\sqrt{d_1 d_2}})\fnorm{\Delta_\U}\fnorm{\Delta_\V}\\
\le &(1+\delta)\fnorm{\Delta}^4 + \frac{\sigstarr}{4}\fnorm{\Delta}^2
\le \sigstarr(\frac{2}{9} + \frac{\delta}{60})\fnorm{\Delta}^2 \le \frac{\sigstarr}{20}\fnorm{\Delta}^2
\end{align*}
Thus by $\fnorm{\Delta_\U}^2  \le \fnorm{\Delta}^2 \le\sigstarr/40$ and $\fnorm{\Delta_\V}^2  \le \fnorm{\Delta}^2 \le\sigstarr/40$,
\begin{align*}
\Delta\Delta\trans: \H:\Delta\Delta\trans = \frac{4}{p}\norm{\Delta_\U\Delta_\V\trans}^2_{\Omega}
+ \left(\norm{\Delta_\U\Delta_\U\trans}_F^2 + \norm{\Delta_\V\Delta_\V\trans}_F^2 - 2\norm{\Delta_\U\Delta_\V\trans}_F^2) \right)
\le \frac{1}{4}\sigstarr\fnorm{\Delta}^2
\end{align*}
This gives:
\begin{align*}
&\Delta\Delta\trans:\H:\Delta\Delta\trans - 3(\N - \N^\star):\H:(\N - \N^\star) \\
=& \Delta\Delta\trans:\H:\Delta\Delta\trans - 3(\W^\star\Delta\trans + \Delta\W^\star{}\trans + \Delta\Delta\trans):\H:(\W^\star\Delta\trans + \Delta\W^\star{}\trans+\Delta\Delta\trans) \\
\le& -12 (\W^\star\Delta\trans:\H:\Delta\Delta\trans + \W^\star\Delta\trans:\H:\W^\star\Delta\trans)\\
\le &-\frac{12}{p} \sqrt{\W^\star\Delta\trans: \H:\W^\star\Delta\trans}(\sqrt{\W^\star\Delta\trans: \H:\W^\star\Delta\trans} - \sqrt{\Delta\Delta\trans: \H:\Delta\Delta\trans})
 \\
\le& -12\sqrt{1-2\delta}(\sqrt{1-2\delta} - \sqrt{1/4})\sigstarr\fnorm{\Delta}^2
\le-1.2\sigstarr\fnorm{\Delta}^2
\end{align*}
The last inequality is by choosing large enough $p$, we have small $\delta$.

~

\noindent\textbf{Case 2:} $\fnorm{\Delta}^2 \ge \sigstarr/40$, by Lemma \ref{lem:global_mc} with high probability, our choice of $p$ gives:
\begin{align*}
\frac{1}{p}\norm{\Delta_\U\Delta_\V\trans}^2_{\Omega} \le& \fnorm{\Delta_\U\Delta_\V\trans}^2 + O\left(\frac{dr\log d}{p}\norm{\Delta_\U\Delta_\V\trans}^2_{\infty}  + \sqrt{\frac{dr\log d}{p}} \fnorm{\Delta_\U\Delta_\V\trans}\norm{\Delta_\U\Delta_\V\trans}_{\infty}\right) \\
\le& \fnorm{\Delta_\U\Delta_\V\trans}^2 +O\left(\frac{dr\log d}{p} \cdot \frac{\mu^4 r^5 (\cn^\star)^4 (\sigstarl)^2}{d_1 d_2} 
+ \sqrt{\frac{dr\log d}{p} \cdot \frac{\mu^4 r^5 (\cn^\star)^4 (\sigstarl)^2}{d_1 d_2}} \fnorm{\Delta}^2 \right) \\
\le& \fnorm{\Delta_\U\Delta_\V\trans}^2 + \frac{(\sigstarr)^2}{1000} + \frac{\sigstarr}{1000}\fnorm{\Delta}^2 \le \fnorm{\Delta_\U\Delta_\V\trans}^2 +0.01\sigstarr\fnorm{\Delta}^2
\end{align*}
Again by Lemma \ref{lem:global_mc} with high probability
\begin{align*}
\frac{1}{p}\norm{\M - \M^\star}^2_{\Omega} \ge& \fnorm{\M-\M^\star}^2 - O\left(\frac{dr\log d}{p}\norm{\M-\M^\star}^2_{\infty}  + \sqrt{\frac{dr\log d}{p}} \fnorm{\M-\M^\star}\norm{\M-\M^\star}_{\infty}\right) \\
\ge& \fnorm{\M-\M^\star}^2 -O\left(\frac{dr\log d}{p} \cdot \frac{\mu^4 r^5 (\cn^\star)^4 (\sigstarl)^2}{d_1 d_2} 
+ \sqrt{\frac{dr\log d}{p} \cdot \frac{\mu^4 r^5 (\cn^\star)^4 (\sigstarl)^2}{d_1 d_2}}  \fnorm{\M-\M^\star}\right) \\
\ge& \fnorm{\M-\M^\star}^2 - \frac{(\sigstarr)^2}{1000} - \frac{\sigstarr}{1000}\fnorm{\M-\M^\star}
\ge 0.99\fnorm{\M-\M^\star}^2 -0.01\sigstarr\fnorm{\Delta}^2
\end{align*}
Then by simple calculation, this gives:
\begin{align*}
&\Delta\Delta\trans:\H:\Delta\Delta\trans - 3(\N - \N^\star):\H:(\N - \N^\star) \\
\le& \fnorm{\Delta\Delta\trans}^2 +0.04\sigstarr\fnorm{\Delta}^2 - 3(0.98\fnorm{\N-\N^\star}^2 -0.04\sigstarr\fnorm{\Delta}^2) \\
\le& -0.94\fnorm{\N-\N^\star}^2 + 0.12\sigstarr\fnorm{\Delta}^2
\le -0.3\sigstarr\fnorm{\Delta}^2
\end{align*}
where the last step is by Lemma \ref{lem:bound}. This finishes the proof.
\end{proof}

Finally we bound the contribution from regularizer (analogous to Lemma \ref{lem:reg_mc}). In fact since our regularizers are very similar we can directly use the same calculation.

\begin{lemma}\label{lem:reg_mc_asym}
By choosing $\alpha_1^2 = \Theta(\frac{\mu r\sigstarl}{d_1}) $, $\alpha_1^2 = \Theta(\frac{\mu r\sigstarl}{d_1}) $, $\lambda_1 \alpha_1^2 \le  O(\sigstarr)$ and $\lambda_2 \alpha_2^2 \le  O(\sigstarr)$, we have:
\begin{equation*}
\frac{1}{4}[\Delta : \hess Q(\W) : \Delta - 4 \la \grad Q(\W), \Delta \ra] \le 0.1\sigstarr \fnorm{\Delta}^2 
\end{equation*}
\end{lemma}
\begin{proof}
By same calculation as the proof of Lemma \ref{lem:reg_mc}, we can show:
\begin{align*}
\frac{1}{4}[\Delta_\U : \hess Q_1(\U) : \Delta_\U - 4 \la \grad Q_1(\U), \Delta_\U \ra] \le 0.1\sigstarr \fnorm{\Delta_\U}^2 \\
\frac{1}{4}[\Delta_\V : \hess Q_2(\V) : \Delta_\V - 4 \la \grad Q_2(\V), \Delta_\V \ra] \le 0.1\sigstarr \fnorm{\Delta_\V}^2 
\end{align*}
Given $Q(\W) = Q_1(\U) + Q_2(\V)$, the lemma follows.
\end{proof}

Combining three lemmas, our main result for asymmetric matrix completion easily follows.
\begingroup
\def\thetheorem{\ref{thm:main_mc_asym}}
\begin{theorem}
Let $d = \max\{d_1, d_2\}$, when sample rate $p \ge \Omega(\frac{\mu^4 r^6 (\cn^\star)^6 \log d}{\min\{d_1, d_2\}})$, choose $\alpha_1^2 = \Theta(\frac{\mu r\sigstarl}{d_1}), \alpha_2^2 = \Theta(\frac{\mu r\sigstarl}{d_2})$ and $\lambda_1 = \Theta(\frac{d_1}{\mu r \cn^\star}), \lambda_2 = \Theta(\frac{d_2}{\mu r \cn^\star})$. With probability at least $1-1/\poly(d)$, for Objective Function \eqref{eq:completion-asymmetric} we have
1) all local minima satisfy $\U\V^\top = \M^\star$ 2) The objective is $(\epsilon, \Omega(\sigstarr), O(\frac{\epsilon}{\sigstarr}))$-strict saddle for polynomially small $\epsilon$.
\end{theorem}
\addtocounter{theorem}{-1}
\endgroup

\begin{proof}
Same argument as the proof Theorem \ref{thm:main_mc} by combining Lemma \ref{lem:incoherence_asym}, \ref{lem:step2_mc_asym} and \ref{lem:reg_mc_asym}.
\end{proof}

\subsection{Robust PCA}
For robust PCA, again a crucial step is to analyze the matrix factorization problem. We prove the following Lemma (analogous to Lemma~\ref{lem:mf_strictsaddle}).

\begin{lemma}
\label{lem:mf_strictsaddle_asym}
Let matrix factorization objective to be ($\U \in \R^{d_1 \times r},\V\in\R^{d_2\times r}$):
\begin{equation*}
f(\W) = 2\fnorm{\U\V\trans - \A}^2 + \frac{1}{2}\fnorm{\U\trans\U - \V\trans \V}^2
\end{equation*}
and $\sigma_r(\A)\ge 30\sigma_{r+1}(\A)$.
then 1) all local minima satisfies $\U\V\trans$ is the top-$r$ SVD of matrix $\A$; 
2) objective is $(\epsilon, 0.2 \sigstarr, \frac{20\epsilon}{\sigstarr})$-strict saddle
\end{lemma}
\begin{proof}
Denote $\M^\star = \proj_r(\A)$ to be the top-$r$ SVD of $\A$, and $\S = \A - \M^\star$ to be the remaining part.
In our framework, we can also view this remaining part as regularization term. That is:
\begin{equation*}
\M:\H_0:\M = \fnorm{\M}^2 \quad \text{and} \quad Q(\W) = 4\la \M^\star-\U\V\trans, \S\ra + 2\fnorm{\S}^2
\end{equation*}
Moreover, since the eigenspace of $\M^\star$ is perpendicular to $\S$, we have:
\begin{align*}
[\Delta : \hess Q(\W) : \Delta - 4 \la \grad Q(\W), \Delta \ra] =& -8 \la\Delta_\U\Delta_\V\trans, \S\ra + 16\la \U\Delta_\V\trans + \Delta_\U\V\trans, \S\ra\\
=& 24\la \U\V\trans, \S\ra \le 24\norm{\S}\fnorm{\Delta_\U}\fnorm{\Delta_\V}
\le 12\norm{\S}\fnorm{\Delta}^2
\end{align*}
The last step is because suppose $\X\D\Y\trans$ is the SVD of $\S$, then 
\begin{equation*}
\la \U\V\trans, \S\ra \le \norm{\D}\fnorm{\X\trans \U}\fnorm{\Y\trans \V}
= \norm{\D}\fnorm{\X\trans (\U-\U^\star)}\fnorm{\Y\trans (\V-\V^\star)} \le \norm{\S}\fnorm{\Delta_\U}\fnorm{\Delta_\V}
\end{equation*}
Using Lemma \ref{lem:sym_to_asym}, the remaining argument is the same as Lemma \ref{lem:mf_strictsaddle}.
\end{proof}

Next we prove when $\W$ is close to the optimal solution of the matrix factorization problem, it must also be close to the true $\W^\star$. The proof of this lemma uses several crucial properties in choice of $\S$ and the sparse set $\mathcal{S}_{\gamma \alpha}$. This will require several supporting lemmas which we prove after the main theorem. Our proof is inspired by \citet{yi2016fast}.

\begin{lemma}\label{lem:mf_to_rpca_asym}
There is an absolute constant $c$, assume $\gamma > c$, and $\gamma\alpha\cdot\mu r\cdot (\cn^\star)^5 \le \frac{1}{c}$. Let $\X^\dagger\D^\dagger \Y^\dagger{}\trans$ be the best rank $r$-approximation of $\M^\star+\S^\star -\S_\W$, where $\S_\W = \argmin_{\S\in \mathcal{S}_{\gamma\alpha}}\norm{\U\V\trans + \S - \M^\star - \S^\star}_F^2$. Let $\U^\dagger = \X^\dag (\D^\dag)^{\frac{1}{2}}$, $\V^\dag = \Y^\dag (\D^\dag)^{\frac{1}{2}}$.  Assume
$\min_{\mR^\top \mR = \mR\mR^\top = \I}\fnorm{\W-\W^\dagger\mR} \le \epsilon$.
Let $\Delta$ be defined as in Definition \ref{def:asymmetricquantities}, then $\fnorm{\Delta} \le O(\epsilon\sqrt{\cn^\star})$ for polynomially small $\epsilon$.
\end{lemma}
\begin{proof}
By assumption, we have $\fnorm{\W - \W^\dagger} \le \epsilon$, we also have $\W$ in the neighborhood of $\W^\star$.

First, we know
\begin{align*}
\fnorm{\N^\dagger - \N} = &\fnorm{\W^\dagger\W^\dagger {}\trans - \W\W\trans}
\le (\norm{\W^\dagger} + \norm{\W})\fnorm{\W -\W^\dagger}\\
\le &(2\norm{\W^\dagger} + \fnorm{\W -\W^\dagger})\fnorm{\W -\W^\dagger}
\le 3 \epsilon \sqrt{\sigstarl}
\end{align*}

Where the last step is due to  $\norm{\U^\dagger}  = \norm{\V^\dagger} \le \sqrt{\norm{\M^\star + \S^\star - \S_{\W}}}
 \le \sqrt{\norm{\M^\star} +\norm{\S^\star - \S_{\W}}} \le\sqrt{ 1.01 \sigstarl}$,
  $\norm{\W^\dagger} \le \norm{\U^\dagger} + \norm{\V^\dagger}$ 
 and $\fnorm{\U -\U^\dagger} \le \frac{20\epsilon}{\sigstarr} \le 0.5 \sqrt{\sigstarl}$ 
 by our choice of $\epsilon$. Then by Lemma \ref{lem:claim1_rpca_asym} and Lemma \ref{lem:N_to_M_asym} we have:
\begin{equation*}
\fnorm{\N^\dagger - \N^\star}  \le 2 \fnorm{\M^\dagger - \M^\star} \le 4 \fnorm{\S_{\W} - \S^\star}
\end{equation*}
By triangle inequality, this gives:
\begin{equation} \label{eq:rpca_claim1}
\fnorm{\N - \N^\star} \le \fnorm{\N^\dagger - \N} + \fnorm{\N^\dagger - \N^\star}  \le 4 \fnorm{\S_{\W} - \S^\star} + 3 \epsilon \sqrt{\sigstarl}
\end{equation}

On the other hand, by Lemma \ref{lem:incoherence_rpca_asym}, we know matrix $\M^\dagger$ is $4\mu (\cn^\star)^4$-incoherent. Thus for any $i\in[d_1]$:
\begin{align*}
\norm{\e_i\trans \U} \le \norm{\e_i\trans(\U -\U^\dagger)} + \norm{\e_i\trans\U^\dagger}
\le \fnorm{\W - \W^\dagger} + 2(\cn^\star)^2 \sqrt{\frac{1.01\mu r \sigstarl}{d_1}}
\le 3(\cn^\star)^2 \sqrt{\frac{\mu r \sigstarl}{d_1}}
\end{align*}
By symmetry, we also have for any $j\in [d_2]$, $\norm{\e_j\trans \V} \le 3(\cn^\star)^2 \sqrt{\frac{\mu r \sigstarl}{d_1}}$. Then, by Lemma \ref{lem:Sopt_rpca_asym}:
\begin{align*}
\fnorm{\S_\W - \S^\star}^2 
\le 2\norm{\M -\M^\star}^2_{\Omega^\star\cup \Omega} + \frac{8}{\gamma - 1}\fnorm{\M -\M^\star}^2
\end{align*}
by Lemma \ref{lem:entries_bound_rpca_asym}:
\begin{equation*}
\norm{\M - \M^\star}^2_\Omega \le 36 \gamma\alpha \mu (\cn^\star)^4 r \norm{\M^\star}(\fnorm{\Delta_\U}^2 + \fnorm{\Delta_\V}^2)
\le 0.04 \sigstarr \fnorm{\Delta}^2
\end{equation*}
Similarly, we also have $\norm{\M - \M^\star}^2_{\Omega^\star} \le 0.04 \sigstarr \fnorm{\Delta}^2$
Clearly, we have $\fnorm{\M - \M^\star} \le \fnorm{\N - \N^\star}$.
By Lemma \ref{lem:aux_deltalinear}, we also have $\sigstarr \fnorm{\Delta}^2 \le \frac{1}{2(\sqrt{2}-1)} \fnorm{\N - \N^\star}^2$. Given our choice of $\gamma$, this gives:
\begin{align} \label{eq:rpca_claim2}
\fnorm{\S_\W - \S^\star}^2 
\le 2\norm{\M -\M^\star}^2_{\Omega^\star\cup \Omega} + \frac{8}{\gamma - 1}\fnorm{\M -\M^\star}^2
\le \frac{1}{25}\fnorm{\N - \N^\star}^2
\end{align}

Finally, combineing Eq.\eqref{eq:rpca_claim1} and Eq.\eqref{eq:rpca_claim2}, we have:
\begin{equation*}
\fnorm{\N^\star - \N} - \frac{60 \epsilon \sqrt{\cn^\star}}{\sqrt{\sigstarr}} \le \frac{4}{5} \fnorm{\N^\star - \N}
\end{equation*}
By Lemma \ref{lem:aux_deltalinear}, we know:
\begin{equation*}
\fnorm{\Delta} \le \frac{1}{\sqrt{\sigstarr}}\sqrt{\frac{1}{2(\sqrt{2}-1)}}\fnorm{\N^\star - \N} \le \sqrt{\frac{1}{2(\sqrt{2}-1)}} \cdot 5 \cdot 3 \epsilon \sqrt{\cn^\star} \le 20 \epsilon \sqrt{\cn^\star}
\end{equation*}
This finishes the proof.
\end{proof}

Now we are ready to prove the main theorem:

\begingroup
\def\thetheorem{\ref{thm:main_rpca_asym}}
\begin{theorem}
There is an absolute constant $c$, if $\gamma > c$, and $\gamma\alpha\cdot\mu r\cdot (\cn^\star)^5 \le \frac{1}{c}$ holds, for objective function Eq.\eqref{eq:rpca-nonconvex} we have 1) all local minima satisfies $\U\V\trans = \M^\star$; 2) objective function is 
$(\epsilon, \Omega(\sigstarr), O(\frac{\epsilon \sqrt{\cn^\star}}{\sigstarr}))$-pseudo strict saddle for polynomially small $\epsilon$.
\end{theorem}
\addtocounter{theorem}{-1}
\endgroup

\begin{proof}
Recall scaled version (multiplied by 4) of objective function Eq.\eqref{eq:rpca-nonconvex} is:
\begin{equation*}
f(\W) = 2\min_{\S\in \mathcal{S}_{\gamma\alpha}}\norm{\U\V\trans + \S - \M^\star - \S^\star}_F^2 
+ \frac{1}{2}\fnorm{\U\trans\U - \V\trans \V}^2
\end{equation*}
Consider point $\W$ with small gradient satisfying $\fnorm{\grad f(\W)} \le \epsilon$.
Let 
\begin{equation*}
\S_\W = \argmin_{\S\in \mathcal{S}_{\gamma\alpha}} \norm{\U\V\trans + \S - \M^\star - \S^\star}_F^2 
\end{equation*}
and function 
$f_{\W}(\tilde{\W}) = 2\norm{\tilde{\U}\tilde{\V}\trans + \S_{\W} - \M^\star - \S^\star}_F^2+ \frac{1}{2}\fnorm{\tilde{\U}\trans\tilde{\U} - \tilde{\V}\trans \tilde{\V}}^2$, then, we know for all $\tilde{\W}$,
we have $f_{\W}(\tilde{\W}) \ge f(\tilde{\W})$ and $f_{\W}(\W) = f(\W)$. 
Since $f_{\W}(\tilde{\W})$ is matrix factorization objective where by Lemma \ref{lem:spectral_infty_rpca_asym}:
\begin{align*}
&\norm{ \S^\star - \S_{\W}} \le 2 \gamma \alpha \cdot 2\frac{\mu r \sigstarl}{d}  \le 0.01\sigstarr \\
&\sigma_{r}(\M^\star + \S^\star - \S_{\W}) \ge \sigstarr - \norm{ \S^\star - \S_{\W}} \ge 0.99 \sigstarr\\
&\sigma_{r+1}(\M^\star + \S^\star - \S_{\W}) \le \norm{ \S^\star - \S_{\W}} \le 0.01 \sigstarr
\end{align*}
This gives $\sigma_{r}(\M^\star + \S^\star - \S_{\W}) \ge 15 \sigma_{r+1}(\M^\star + \S^\star - \S_{\W})$. 
Given $\fnorm{\grad f_{\W}(\W)} = \fnorm{\grad f(\W)} \le \epsilon$, by Lemma \ref{lem:mf_strictsaddle_asym}, we know either $\lambda_{\min}(\hess f_{\W}(\W)) \le -0.2 \sigstarr$ or $\min_{\mR^\top \mR = \mR\mR^\top = \I}\fnorm{\W-\W^\dagger\mR} \le \frac{20\epsilon}{\sigstarr}$ where  $\U^\dagger = \X^\dag (\D^\dag)^{\frac{1}{2}}$, $\V^\dag = \Y^\dag (\D^\dag)^{\frac{1}{2}}$ and $\X^\dagger\D^\dagger \Y^\dagger{}\trans$ is the best rank $r$-approximation of $\M^\star+\S^\star -\S_\W$.  
By Lemma \ref{lem:mf_to_rpca_asym}, we immediately have $\fnorm{\Delta}  \le \frac{10^3\epsilon\sqrt{\cn^\star}}{\sigstarr}$, which finishes the proof.
\end{proof}

\subsection{Supporting Lemmas for robust PCA}

In the proof of Lemma~\ref{lem:mf_to_rpca_asym}, we used several supporting lemmas. We now prove them one by one. The first is a classical result from matrix perturbations.

\begin{lemma}\label{lem:claim1_rpca_asym}
Let $\M^\dagger$ be the top-$r$ SVD of matrix $\M^\star + \S \in \R^{d_1 \times d_2}$ where $\M^\star$ is rank $r$.
Then we have:
\begin{equation*}
\fnorm{\M^\dagger - \M^\star} \le 2\fnorm{\S}
\end{equation*}
\end{lemma}
\begin{proof}
By triangle inequality:
\begin{align*}
\fnorm{\M^\star - \M^\dagger} \le& \fnorm{\M^\star + \S - \M^\dagger} + \fnorm{\S} 
\end{align*}
For the second term, by the fact $\M^\star$ is rank $r$, the definition of $\M^\dagger$ and Weyl's inequality:
\begin{align*}
\fnorm{\M^\star + \S - \M^\dagger}^2 =& \sum_{i={r+1}}^d \sigma_i^2(\M^\star + \S) 
\le \sum_{i={r+1}}^d (\sigma_{r+1}(\M^\star) + \sigma_{i-r}(\S))^2
= \sum_{i={r+1}}^d \sigma^2_{i-r}(\S) \le \fnorm{\S}^2
\end{align*}
This finishes the proof.
\end{proof}


Next we show how to bound the norm of matrix $\M-\M^\star$ restricted to a sparse set $\Omega$.
\begin{lemma}\label{lem:entries_bound_rpca_asym}
Let $\M^\star = \U^\star \V^\star {}\trans$ are both $\mu$-incoherent matrix, 
$\M = \U\V\trans$ satisfies $\max_i\norm{\e_i\trans \U}^2 \le \frac{\mu r \norm{\M^\star}}{d_1}$
and  $\max_j\norm{\e_j\trans \V}^2 \le \frac{\mu r \norm{\M^\star}}{d_2}$
and $\Omega$ has at most $\alpha$ fraction of non-zero entries in each row/column. Then:
\begin{equation*}
\norm{\M - \M^\star}^2_\Omega \le 4 \alpha \mu r \norm{\M^\star}(\fnorm{\Delta_\U}^2 + \fnorm{\Delta_\V}^2)
\end{equation*}
where $\Delta_\U = \U - \U^\star$, $\Delta_\V = \V - \V^\star$.
\end{lemma}
\begin{proof}

Then for any $(i, j)$, since $\M, \M^\star$ are both $\mu$-incoherent, we have:
\begin{align*}
|(\M -\M^\star)_{(i,j)}|  =& |(\U\Delta_\V\trans+\Delta_\U\V^\star{}\trans)_{(i,j)}|
\le \norm{\e_i\trans \U}\norm{\e_j\trans \Delta_\V} + \norm{\e_i\trans \Delta_\U}\norm{\e_j\trans \V^\star}\\
\le & \sqrt{\norm{\M^\star}\frac{\mu r}{d_1}}\norm{\e_j\trans \Delta_\V}  + \sqrt{\norm{\M^\star}\frac{\mu r}{d_2}}\norm{\e_i\trans \Delta_\U}
\end{align*}
Therefore, 
\begin{align*}
\norm{\M^\star - \M}^2_\Omega \le& 2\sum_{(i, j)\in \Omega} \norm{\M^\star}\frac{\mu r}{d_1}\norm{\e_j\trans \Delta_\V}^2  + \norm{\M^\star}\frac{\mu r}{d_2}\norm{\e_i\trans \Delta_\U}^2 \\
\le& 4 \alpha \mu r \norm{\M^\star}(\fnorm{\Delta_\U}^2 + \fnorm{\Delta_\V}^2)
\end{align*}
\end{proof}

Next, we show for any fixed sparse estimator $\S$, if $\M^\star$ is incoherent, the top-$r$ SVD of  $\M^\star + \S^\star - \S$
will also be incoherent. Thus, sparse matrix will not interfere incoherence in this sense.
\begin{lemma}\label{lem:incoherence_rpca_asym}
For any $\S \in \mathcal{S}_{\gamma\alpha}$, let $\M^\dagger$ be the top-$r$ SVD of $\M^\star + \S^\star - \S$, and the SVD of $\M^\dagger$ to be $\X\D\Y\trans$,
if $\gamma\alpha\cdot\mu r\cdot \cn^\star \le \frac{1}{1000}$, then we have:
\begin{equation*}
\max_i\norm{\e_i\trans\X}^2 \le 4\frac{\mu r (\cn^\star)^4}{d_1} 
\quad \text{and} \quad
\max_j\norm{\e_j\trans\Y}^2 \le 4\frac{\mu r (\cn^\star)^4}{d_2} 
\end{equation*}
where condition number $\cn^\star = \sigma_1(\M^\star) / \sigma_r(\M^\star)$.
\end{lemma}

\begin{proof}
Since $\X \D\Y\trans$ is the top $r$ SVD of $\M^\star + \S^\star - \S$, we have:
\begin{align*}
(\M^\star + \S^\star - \S)(\M^\star + \S^\star - \S)\trans\X = \X \D^2
\end{align*}
Therefore, for any $i \in [d]$, because $\S, \S^\star$ has at most $\gamma\alpha$ fraction non-zero entries in each row, and $\norm{\S^\star}_\infty \le 2 \frac{\mu r \sigstarl}{\sqrt{d_1 d_2}}$ and $\S \in \mathcal{S}_{\gamma\alpha}$, we have:
\begin{align*}
\sigma^2_r(\D)\norm{\e_i\trans \X} \le &\norm{\e_i\trans \X \D^2}
= \norm{\e_i\trans(\M^\star + \S^\star - \S)(\M^\star + \S^\star - \S)\trans\X} \\
\le & (\norm{\e_i\trans\M^\star} + \norm{\e_i\trans (\S^\star - \S)}) \norm{\M^\star + \S^\star - \S}\\
\le & \left[\sqrt{\frac{\mu r}{d_1}}\sigstarl + \sqrt{\gamma\alpha d_2}\right] (\norm{\S^\star}_\infty + \norm{\S}_\infty)\norm{\D}\\
\le & \sqrt{\frac{\mu r}{d_1}}\sigstarl \left(1+ 5\sqrt{\gamma\alpha\cdot\mu r}\right)\norm{\D}
\le 1.2 \sqrt{\frac{\mu r}{d_1}}\sigstarl\cdot \norm{\D}
\end{align*}
On the other hand, by Lemma \ref{lem:spectral_infty_rpca_asym} and Weyl's inequality, we also have
\begin{align*}
\norm{\S^\star - \S} \le&  \norm{\S^\star} + \norm{\S} \le \gamma \alpha \sqrt{d_1 d_2} (\norm{\S^\star}_\infty + \norm{\S}_\infty)  \le 0.1 \sigstarr\\
\norm{\D} \le & \norm{\M} + \norm{\S^\star - \S} \le 1.1 \sigstarl \\
\sigma_r(\D) \ge& \sigma_r(\M^\star) - \norm{\S^\star - \S} \ge 0.9 \sigstarr
\end{align*}
Therefore, in sum, we have:
\begin{align*}
\norm{\e_i\trans \X} \le  1.2 \sqrt{\frac{\mu r}{d_1}}\frac{\sigstarl\cdot \norm{\D}}{\sigma^2_r(\D)} \le 2\sqrt{\frac{\mu r}{d_1}} (\cn^\star)^2
\end{align*}
By symmetry, we can also prove it for $\norm{\e_j\trans \Y}$ with any $j\in [d_2]$.
\end{proof}


We also need to upper bound the spectral norm of sparse matrix $\S$.

\begin{lemma}\label{lem:spectral_infty_rpca_asym}
For any sparse matrix $\S$ that can only has at most $\alpha$ fraction non-zero entries in each row/column, we have:
\begin{equation*}
\norm{\S} \le \alpha\sqrt{d_1 d_2}\norm{\S}_\infty
\end{equation*}
\end{lemma}
\begin{proof}
Let $\Omega$ be the support of matrix $\S$, and $\beta = \sqrt{\frac{d_1}{d_2}}$, we have:
\begin{align*}
\norm{\S} = \sup_{(\x, \y): \norm{\x} = 1, \norm{\y} =1 } \x\trans\S\y
= \sum_{(i, j) \in \Omega} x_i (\S)_{ij} y_j
\le \frac{1}{2}\sum_{(i, j) \in \Omega}\norm{\S}_\infty (\beta x_i^2 + \frac{1}{\beta} y_j^2)
\le \alpha \sqrt{d_1 d_2}\norm{\S}_\infty
\end{align*}

\end{proof}

Then, we show a crucial property of the optimal $\S\in \mathcal{S}_{\gamma\alpha}$:

\begin{lemma}\label{lem:maxflow_asym}
For any matrix $\A$, let $\Sopt = \argmin_{\S\in \mathcal{S}_{\gamma\alpha}}\fnorm{\A - \S}^2$, and $\Omega$ be the support of $\S$, then for any $(i, j) \in [d_1]\times[d_2] - \Omega$, we have:
\begin{equation*}
|\A_{(i,j)}| \le |\A^{(\gamma \alpha d_1)}_{(i, \cdot)}| + |\A^{(\gamma \alpha d_2)}_{(\cdot, j)}|
\end{equation*}
where $\A^{(k)}_{(i, \cdot)}$ is the $k$-th largest element (in terms of absolute value) in $i$-th row of $\A$ and $\A^{(k)}_{(\cdot, j)}$ the $k$-th largest element (in terms of absolute value) in $j$-th column of $\A$.
\end{lemma}
\begin{proof}
Assume the contradiction that in optimal solution $\S$ there is a pair $(i,j)\in [d_1]\times [d_2]-\Omega$ such that 
\begin{equation*}
|\A_{(i,j)}| > |\A^{(\gamma \alpha d_1)}_{(i, \cdot)}| + |\A^{(\gamma \alpha d_2)}_{(\cdot, j)}|
\end{equation*}

If row $i$ has exactly $\gamma \alpha d_1$ elements in $\Omega$, let $e_1 = (i,j')$ be the smallest entry in row $i$ ($j' = \arg\min_{z:(i,z)\in \Omega} |\A_{(i,z)}|$), clearly $\A_{(i,j')} \le \A^{(\gamma\alpha d_1)}_{(i,\cdot)}$. If row $i$ has fewer elements we just let $e_1$ be empty. Similarly, if column $j$ has exactly $\gamma \alpha d_2$ elements in $\Omega$, let $e_2 = (i',j)$ be the smallest entry in the column ($i'=\arg\min_{z:(z,j)\in \Omega} |\A_{(z,j)}|$), we also have $\A_{(i',j)} \le \A^{(\gamma\alpha d_2)}_{(\cdot,j)}$.

Now we can add $(i,j)$ to $\Omega$, and remove $e_1$ and $e_2$. Call the resulting matrix $\S'$. This clearly does not violate the support constraint. Let $q(x) = (x - \max\{0,x-2\frac{\mu r\sigma_1^\star}{\sqrt{d_1d_2}}\})^2$, 
this function is monotone for $x>0$ and satisfies $q(x)+q(y) \le q(x+y)$. Now the difference we get from changing $\S$ to $\S'$ is 
\begin{equation*}
\|\A-\S'\|_F^2 = \|\A-\S\|_F^2 - q(|\A_{(i,j)}|)+q(|\A_{e_1}|)+q(|\A_{e_2}|) <\|\A-\S\|_F^2.
\end{equation*}

This contradicts with the fact that $\S$ was the optimal. Therefore there cannot be such an entry $(i,j)$.
\end{proof}

Finally, by using above lemmas, we can show how to bound the difference between optimal $\S$ and true sparse matrix $\S^\star$, 
which is a key step in Lemma \ref{lem:mf_to_rpca_asym}.

\begin{lemma}\label{lem:Sopt_rpca_asym}
For any matrix $\A \in \R^{d_1 \times d_2}$, let $\S^\star \in \mathcal{S}_{\alpha}$
and 
\begin{equation*}
\Sopt = \argmin_{\S\in \mathcal{S}_{\gamma\alpha}}\norm{\S - \S^\star + \A}_F^2
\end{equation*}
Let $\Omega$ be the support of $\Snew$ and $\Omega^\star$ be the support of $\S^\star$, we will have:
\begin{align*}
\fnorm{\Snew - \S^\star}^2 
\le 2\norm{\A}^2_{\Omega^\star\cup \Omega} + \frac{8}{\gamma - 1}\fnorm{\A}^2
\end{align*}
\end{lemma}
\begin{proof}
Clearly, the support of $\Snew - \S^\star$ must be a subset of $\Omega \cup \Omega^\star$. Therefore, we have:
\begin{equation*}
\fnorm{\Snew - \S^\star}^2 = \norm{\Snew - \S^\star}^2_\Omega + \norm{\Snew - \S^\star}^2_{\Omega^\star - \Omega}
\end{equation*}
For the first term, since $\Snew$ is defined as minimizer over $\mathcal{S}_{\gamma\alpha}$, we know for $(i, j) \in \Omega$:
\begin{equation*}
(\Snew)_{(i,j)} = \max\left\{\min\left\{(\S^\star- \A)_{(i, j)},  ~2\frac{\mu r \sigstarl}{\sqrt{d_1 d_2}} \right\}, ~-2\frac{\mu r \sigstarl}{\sqrt{d_1 d_2}}\right\}
\end{equation*}
By assumption we know $\S^\star \in \mathcal{S}_{\alpha}$ thus $\norm{\S^\star}_\infty \le 2 \frac{\mu r \sigstarl}{\sqrt{d_1 d_2}}$, this gives
$|(\Snew - \S^\star)_{(i,j)}| \le |(\A)_{(i, j)}|$ thus:
\begin{equation*}
\norm{\Snew - \S^\star}^2_\Omega \le \norm{\A}^2_\Omega
\end{equation*}
For the second term, by triangle inequality, we have:
\begin{align*}
\norm{\Snew - \S^\star}_{\Omega^\star - \Omega} = \norm{\S^\star}_{\Omega^\star - \Omega}
\le \norm{\S^\star - \A}_{\Omega^\star - \Omega} + \norm{\A}_{\Omega^\star - \Omega}
\end{align*}
By Lemma \ref{lem:maxflow_asym}, we have for any $(i, j) \in \Omega^\star - \Omega$:
\begin{align*}
|(\S^\star - \A)_{(i, j)}|^2 \le &
2|(\S^\star - \A)^{(\gamma \alpha d_1)}_{(i, \cdot)}|^2 + 2|(\S^\star - \A)^{(\gamma \alpha d_2)}_{(\cdot, j)}|^2 \\
\le & 2|(\A)^{((\gamma-1) \alpha)}_{(i, \cdot)}|^2 + 2|(\A)^{((\gamma-1) \alpha)}_{(\cdot, j)}|^2 \\
\le & \frac{2}{(\gamma-1) \alpha}\left(\norm{\e_i\trans \A}^2 + \norm{ \A\e_j}^2\right)
\end{align*}
where the second inequality used the fact that $\S^\star$ has at most $\alpha$ non-zero entries each row/column. Then, we have:
\begin{align*}
\norm{\S^\star - \A}^2_{\Omega^\star - \Omega} \le& \sum_{(i,j) \in\Omega^\star - \Omega }\frac{2}{(\gamma-1) \alpha}\left(\norm{\e_i\trans \A}^2 + \norm{ \A\e_j}^2\right) \\
\le& \frac{4}{\gamma - 1}\fnorm{\A}^2
\end{align*}
Therefore, in conclusion, we have:
\begin{align*}
\fnorm{\Snew - \S^\star}^2 \le &\norm{\Snew - \S^\star}^2_\Omega + \norm{\Snew - \S^\star}^2_{\Omega^\star - \Omega}\\
\le &\norm{\A}^2_\Omega + 2\norm{\A}^2_{\Omega^\star - \Omega}
+ 2\norm{\S^\star - \A}^2_{\Omega^\star - \Omega}\\
\le & 2\norm{\A}^2_{\Omega^\star\cup \Omega} + \frac{8}{\gamma - 1}\fnorm{\A}^2
\end{align*}
\end{proof}

\section{Matrix Sensing with Noise}
\label{sec:noise}

In this section we demonstrate how to handle noise using our framework. The key idea here is to consider the noise as a perturbation to the original objective function and use $Q(\U)$ (originally the regularizer) to also capture the noise.

\subsection{Symmetric case}
Here, we assume in each observation, instead of observing the exact value $b_i = \la \A_i, \M^\star\ra$ we observe $b_i = \la \A_i,\M^\star\ra + n_i$. Here $n_i$ is i.i.d $\mathcal{N}(0,\sigma^2)$. Recall the objective function in Section~\ref{sec:symmetric}, we now have:
\begin{equation}
f(\M) = \frac{1}{m}\sum_{i=1}^m (\la \M-\M^\star,\A_i\ra+n_{i})^2
\end{equation}

Define $Q(\U) = f(\U\U^\top)- \frac{1}{2}(\U\U^\top - \M^\star):\H:(\U\U^\top - \M^\star)$ to be the perturbation, we can write out the non-convex objective
\begin{equation}
\min_{\U\in\R^{d\times r}} \frac{1}{2}(\U\U^\top - \M^\star):\H:(\U\U^\top - \M^\star) + Q(\U).\label{eq:noisy-symmetric}
\end{equation}

Now we can use the same framework. Again, since for matrix sensing we have the RIP property, we do not need the first step to restrict to special low rank matrices. Using our approach we can get

\begin{theorem}
For objective Equation \eqref{eq:noisy-symmetric}, suppose the sensing matrices $\{\A_i\}$'s satisfy $(2r,1/10)$-RIP, with high probability all points satisfy first and second order optimality condition must satisfy
$$
\norm{\U\U^\top - \M^\star}_F\leq O(\sigma\sqrt{\frac{dr\log m}{m}}).
$$
\end{theorem}

\begin{proof}
Using the same proof as Theorem~\ref{thm:sensing-symmetric-main}, we know
$$
\Delta\Delta\trans:\H:\Delta\Delta\trans - 3(\N - \N^\star):\H:(\N - \N^\star) \ge -0.5 \|\N-\N^\star\|_F^2.
$$

We then bound the contribution from $Q$.
\begin{align*}
Q(\U)=&-\frac{2}{m}\sum_{i=1}^m (\la \M-\M^\star,\A_i\ra n_{i}) +\frac{1}{m}\sum_{i=1}^m(n_i)^2\\
\la \grad Q(\U), \Delta \ra=&-\frac{2}{m}\sum_{i=1}^m(\la \U\Delta\trans + \Delta\U\trans,\A_i\ra n_i)\\
\Delta : \hess Q(\U) :\Delta=&-\frac{4}{m}\sum_{i=1}^m (\la \Delta\Delta\trans,\A_i\ra n_i)
\end{align*}

Therefore,
\begin{align*}
&[\Delta : \hess Q(\U) : \Delta - 4 \la \grad Q(\U), \Delta \ra]\\
\leq & -\frac{4}{m}\sum_{i=1}^m (\la \Delta\Delta\trans,\A_i\ra n_i)+\frac{8}{m}\sum_{i=1}^m(\la \U\Delta\trans + \Delta\U\trans,\A_i\ra n_i)\\
=&\frac{4}{m}\sum_{i=1}^m (\la \M - \M^\star,\A_i\ra \cdot n_i)+\frac{4}{m}\sum_{i=1}^m(\la \U\Delta\trans + \Delta\U\trans,\A_i\ra \cdot n_i).
\end{align*}
Here the last step follows from $\M  - \M^\star +\Delta\Delta\trans=\U\Delta\trans + \Delta\U\trans$. Intuitively, $n_i$ is random and should not have large correlation with any fixed vector. We formalize this in Lemma~\ref{lem:sensing_noise}. Using this lemma, we know
\begin{align*}
|\frac{4}{m}\sum_{i=1}^m (\la \M - \M^\star,\A_i\ra n_i)|& \leq 4\sigma\sqrt{\frac{dr\log m}{m}}\norm{\M - \M^\star}_F\\
\\
|\frac{4}{m}\sum_{i=1}^m(\la \U\Delta\trans + \Delta\U\trans,\A_i\ra n_i| & \leq 4\sigma\sqrt{\frac{dr\log m}{m}}\norm{\U\Delta\trans + \Delta\U\trans}_F\\
& \le 4(1+\sqrt{2})\sigma\sqrt{\frac{dr}{m}}\norm{\M - \M^\star}_F
\end{align*}
Here the last inequality follows from $\norm{\U\Delta\trans + \Delta\U\trans}_F \leq \norm{\M  - \M^\star}_F + \norm{\Delta\Delta\trans}_F \leq (1+\sqrt{2})\norm{\M  - \M^\star}_F$ (by Lemma~\ref{lem:aux_delta2}). Now using the main Lemma~\ref{lem:main}, we know
\begin{align*}
\Delta : \hess f(\U) :\Delta
\leq &-\frac{1}{2}\norm{\M - \M^\star}_F^2\\
&+(8+4\sqrt{2})\sigma\sqrt{\frac{dr\log m}{m}}\norm{\M - \M^\star}_F.
\end{align*}
If the current point satisfy the second order optimality condition we must have
\begin{align*}
\norm{\M - \M^\star}_F \leq O(\sigma\sqrt{\frac{dr\log m}{m}}).
\end{align*}
\end{proof}

Note that this bound matches the intuitive bound from the VC-dimension of rank-$r$ matrices.


\subsection{Asymmetric case}

For the asymmetric case, the proof is again almost identical. We use the same noise model where the observation $b_i = \la \A_i, \M^\star\ra + n_i$ where $n_i\sim N(0,\sigma^2)$. We also use the same notations as in Definition~\ref{def:asymmetricquantities}.
Let $Q(\W) = Q(\U,\V) = \frac{2}{m} \sum_{i=1}^m[(\la \M-\M^\star,\A_i\ra+n_{i})^2 - (\la \M-\M^\star,\A_i\ra)^2]$, we have the objective function
\begin{align}
f(\U, \V) =& \frac{1}{2}(\N-\N^\star):\H:(\N-\N^\star) + Q(\W).\label{eq:noisy-asymmetric}
\end{align}

Again by bounding the gradient and Hessian for $Q(\W)$ we get the following

\begin{theorem}
For objective Equation \eqref{eq:noisy-asymmetric}, suppose the sensing matrices $\{\A_i\}$'s satisfy $(2r,1/20)$-RIP, let $d=d_1+d_2$, with high probability all points satisfy first and second order optimality condition must satisfy
$$
\norm{\U\V^\top - \M^\star}_F\leq O(\sigma\sqrt{\frac{dr\log m}{m}}).
$$
\end{theorem}

\begin{proof}
Again using the same proof as Theorem~\ref{thm:sensing-main}, we know
$$
\Delta\Delta\trans:\H:\Delta\Delta\trans - 3(\M - \M^\star):\H:(\M - \M^\star) \ge -0.5 \|\M-\M^\star\|_F^2.
$$
We then bound the contribution from $Q$.
\begin{align*}
Q(\W)=&-\frac{8}{m}\sum_{i=1}^m (\la \M-\M^\star,\A_i\ra n_{i}) +\frac{4}{m}\sum_{i=1}^m(n_i)^2\\
\la \grad Q(\W), \Delta \ra=&-\frac{8}{m}\sum_{i=1}^m(\la \U\Delta_V\trans + \Delta_U\V\trans,\A_i\ra n_i)\\
\Delta : \hess Q(\W) :\Delta=&-\frac{16}{m}\sum_{i=1}^m (\la \Delta_U\Delta_V\trans,\A_i\ra n_i)
\end{align*}

Let $\B_i$ be the $(d_1+d_2)\times (d_1+d_2)$ matrix whose diagonal blocks are 0, and off diagonal blocks are equal to $\A_i$ and $\A_i^\top$ respectively, we have
\begin{align*}
&[\Delta : \hess Q(\W) : \Delta - 4 \la \grad Q(\W), \Delta \ra]\\
\leq& -\frac{16}{m}\sum_{i=1}^m (\la \Delta\Delta\trans,\B_i\ra n_i)+\frac{32}{m}\sum_{i=1}^m(\la \W\Delta\trans + \Delta\W\trans,\B_i\ra n_i)\\
=&\frac{16}{m}\sum_{i=1}^m (\la \N - \N^\star,\B_i\ra n_i)+\frac{16}{m}\sum_{i=1}^m(\la \W\Delta\trans + \Delta\W\trans,\B_i\ra n_i).
\end{align*}
Now we can use Lemma~\ref{lem:sensing_noise} again to bounding the noise terms: 
\begin{align*}
|\frac{8}{m}\sum_{i=1}^m (\la \N - \N^\star,\B_i\ra n_i)|& \leq 8\sigma\sqrt{\frac{dr\log m}{m}}\norm{\M - \M^\star}_F\\
\\
|\frac{8}{m}\sum_{i=1}^m(\la \W\Delta\trans + \Delta\W\trans,\B_i\ra n_i| & \leq 8\sigma\sqrt{\frac{dr\log m}{m}}\norm{\U\Delta_V\trans +\Delta_U\V\trans}_F\\
& \le 8\sigma\sqrt{\frac{dr\log m}{m}}\norm{\W\Delta\trans +\Delta\W\trans}_F\\
& \leq 8(1+\sqrt{2})\sigma\sqrt{\frac{dr}{m}}\norm{\N-\N^\star}).
\end{align*}

Therefore the Hessian at $\Delta$ direction is equal to:
\begin{align*}
\Delta : \hess f(\W) :\Delta
\leq &-\frac{1}{2}\norm{\N - \N^\star}_F^2
+(16+8\sqrt{2})\sigma\sqrt{\frac{dr\log m}{m}}\norm{\N-\N^\star}_F.
\end{align*}
When the point satisfies the second order optimality condition we have
\begin{align*}
\norm{\N-\N^\star}_F \leq O(\sigma\sqrt{\frac{dr\log m}{m}}).
\end{align*}
In particular, $\M-\M^\star$ is a submatrix of $\N-\N^\star$, therefore $\norm{\M-\M^\star}_F\le O(\sigma\sqrt{\frac{dr\log m}{m}})$.
\end{proof}

\section{Proof Sketch for Running Time}
\label{sec:runtimesketch}
In this section we sketch the proof for Corollary~\ref{cor:runtimeformal}. 


\begingroup
\def\thetheorem{\ref{cor:runtimeformal}}
\begin{corollary}
Let $R$ be the Frobenius norm of the initial points $\U_0, \V_0$, a saddle-avoiding local search algorithm can find a point $\epsilon$-close to global optimal for  matrix sensing \eqref{eq:sensing-symmetric}\eqref{eq:sensing-asymmetric}, matrix completion \eqref{eq:completion-symmetric}\eqref{eq:completion-asymmetric} in $\mbox{poly}(R, 1/\epsilon,d,\sigstarl,1/\sigstarr)$ iterations. For robust PCA \eqref{eq:rpca-symmetric}\eqref{eq:rpca-nonconvex}, alternating between a saddle-avoiding local search algorithm and computing optimal $\S\in\mathcal{S}_{\gamma\alpha}$ will find a point $\epsilon$-close to global optimal in $\mbox{poly}(R, 1/\epsilon,d,\sigstarl,1/\sigstarr)$ iterations.
\end{corollary}
\addtocounter{theorem}{-1}
\endgroup

The full proof require some additional analysis depending on the particular algorithm used, and is highly dependent on the detailed proofs of the guarantees, so we only give a proof sketch here.

Our geometric results show that for small enough $\epsilon'$, the objective functions are $(\epsilon', \gamma, C\epsilon')$ strict-saddle where $\gamma$ and $C$ may depend polynomially on $(\sigstarl, \sigstarr)$. Choose $\epsilon' = \epsilon/C$, we know for each point, either it has a gradient at least $\epsilon'$, or the Hessian has an eigenvalue smaller than $-\gamma$, or $\|\Delta\|_F \le \epsilon$. In the first two cases, by Definition~\ref{def:saddle-avoid} we know saddle-avoiding algorithm can decrease the function value by an inverse polynomial factor in polynomial time. By the radius of the initial solution, the difference in function value between the original solution and optimal solution is bounded by $\mbox{poly}(R,\sigstarl,d)$, so after a polynomial number of iterations we can no longer decrease function value and must be in the third case (where $\|\Delta\|_F \le \epsilon$).

\paragraph{Smoothness and Hessian Lipschitz} The objective funcitons we work with are mostly polynomials thus both smooth and Hessian Lipschitz. The regularizers we add also tried to make sure at least both smoothness and Hessians Lipschitz are satisfied. However, the objective functions 
are still not very smooth or Hessian-Lipschitz especially in the region when then the norm of $(\U,\V)$ is very large. This is because the polynomials are of degree more than 2 and in general the smoothness and Hessian-Lipschitzness parameters ($l,\rho$) depend on the norm of the current point $(\U,\V)$. It is not hard to show that when the solution is constrained into a ball of radius $R$, the parameters $l,\rho$ are all $\mbox{poly}(R)$. Therefore to complete the proof we need to show that the intermediate steps of the algorithms cannot escape from a large ball. In fact, for all the known algorithms, on our objective functions the following is true

\begin{lemma}
For current saddle avoiding algorithms (including cubic regularization \citep{nesterov2006cubic}, perturbed gradient descent \citep{jin2017escape} 
) There exists a radius $R$ that is polynomial in problem parameters, such that if initially $\fnorm{\U}+ \fnorm{\V} = R_0 \le R$, then with high probability all the iterations will have $\fnorm{\U}+ \fnorm{\V} \le 2R$.
\end{lemma}

The proof of this lemma is mostly calculations (and observing the fact that when $\U,\V$ are both very large, the gradient will essentially point to $0$), as an example this is done for matrix factorization in \citep{jin2017escape}. We omit the proof in this paper.

Note that our geometric results for matrix sensing does not depend on the dimension. If we can prove a bound on $R$ that is independent of the dimension $d$, by recent result in \citep{jin2017escape}, we can get algorithms whose number of iterations depend only on $\log d$ for matrix sensing.

\paragraph{Handling Robust-PCA} For robust PCA, the objective function is only pseudo strict-saddle (see Definition~\ref{def:pseudo_strict_saddle}). 

In order to turn the geometric property to an algorithm, the first observation is that the optimal $\S$ for $\U,\V$ can be found in polynomial time: The problem of finding the optimal $\S$ can be formulated as a weighted bipartite matching problem where one part corresponds to the rows, the other part corresponds to the columns, and the value corresponds to the improvement in objective function when we add $(i,j)$ into the support. According to the definition of $\mathcal{S}$, each row/column can be matched a limited number of times. This problem can be solved by converting it to max-flow, and standard analysis shows that there exists an optimal integral solution.

Next we view the robust PCA objective function of form $f(\U,\V) = \min_{\S\in\mathcal{S}_{\gamma\alpha}} g(\U,\V; \S)$.
We show that alternating between saddle-avoiding local search and optimizing $\S$ over $\mathcal{S}_{\gamma\alpha}$ will allow us to get the desired guarantee. For a point $\U,\V$, if it is not close enough to the global optimal solution, we can fix the optimal $\S$ for $\U,\V$ and study $g(\U,\V; \S)$. First, we know for this optimal choice of $\S$, the gradient of $g(\U,\V; \S)$ over $(\U, \V)$ is the same as gradient of $f(\U, \V)$.
Then, by Theorem~\ref{thm:rpca-symmetric-main} / Theorem~\ref{thm:main_rpca_asym}, we know either the gradient of $g(\U,\V; \S)$  is large or the Hessian of $g(\U,\V; \S)$ has an eigenvalue at most $-\Omega(\sigstarr)$. By the guarantee of saddle-avoiding algorithms in polynomial number of steps we can find $\U',\V'$ such that the objective function $g(\U',\V';\S)$ will decrease by a inverse polynomial. After that, replacing $\S$ with $\S'$ (optimal for $\U',\V'$) cannot increase function value, so in polynomial time we found a new point such that $f(\U',\V') \le f(\U,\V) - \delta$ where $\delta$ is at least an inverse polynomial. This procedure cannot be repeated by more than polynomial number of times (because the function value cannot decrease below the optimal value), so the algorithm finds an approximate optimal point in polynomial time.


\section{Concentrations}

In this section we summarize the concentration inequalities we use for different problems.

\subsection{Matrix Sensing}
\begin{definition}[Restrict Isometry Property] Measurement $\mathcal{A}$ ($\{\A_i\}$) satisfies $(r, \delta_r)$-Restrict Isometry Property (RIP) if for any matrix $\X$ with rank $r$, we have:
\begin{equation*}
(1-\delta_r) \fnorm{\X}^2 \le \frac{1}{m}\sum_{i=1}^m \la \A_i, \X\ra^2 \le (1+\delta_r) \fnorm{\X}^2
\end{equation*}

\end{definition}
In the case of Gaussian measurement, standard analysis shows when $m = O(\frac{dr}{\delta^2})$, we have $\mathcal{A}$ satisfying $(r, \delta)$-RIP condition with probability at least  $1-e^{\Omega (d)}$. (\cite{candes2011tight}, Theorem 2.3)

We need the follow inequality for handling noise.

\begin{lemma}\label{lem:sensing_noise}
Suppose the set of sensing matrices $\A_1,\A_2,...,\A_m$ satisfy the $(2r,\delta)$-RIP condition, let $n_1,n_2,...,n_m$ be iid. Gaussian $N(0,\sigma^2)$, then with high probability for any matrix $\M$ of rank at most $r$, we have
$$
|\frac{1}{m}\sum_{i=1}^m n_i \la \A_i,\M\ra| \le O\left(\sigma \sqrt{\frac{dr\log m}{m}}\|\M\|_F\right).
$$
\end{lemma}

\begin{proof}
Since the LHS is linear in $\M$ we focus on matrices with $\|\M\|_F = 1$.

Let $\mathcal{X}$ be an $\epsilon$-net for rank-$r$ matrices with Frobenius norm 1. By standard constructions we know $\log |\mathcal{X}| \le dr\log (dr/\epsilon)$. We will set $\epsilon = 1/m$ so $\log (dr/\epsilon) = O(\log m)$ ($m$ is at least $dr$ for RIP condition). Now, for any matrix $\M\in \mathcal{X}$, we know $\frac{1}{m}\sum_{i=1}^m n_i \la \A_i,\M$ is just a Gaussian random variable with variance at most $\sigma^2(1+\delta)/m$. Therefore, the probability that it is larger than $\sigma \sqrt{\frac{dr\log m}{m}}$ is at most $\exp(-C'dr\log m)$. When $C$ is a large enough constant we can apply union bound, and we know for every $\M\in \mathcal{X}$, 
$$
|\frac{1}{m}\sum_{i=1}^m n_i \la \A_i,\M\ra| \le O\left(\sigma \sqrt{\frac{dr\log m}{m}}\|\M\|_F\right).
$$
On the other hand, with high probability the norm of the vector $\mathbf{n}\in \R^{m}$ is $O(\sigma\sqrt{m})$. Suppose $\M$ is not in $\mathcal{X}$, let $\M'$ be the closest matrix in $\mathcal{X}$, let $\z_i = \la \A_i, \M-\M'\ra$, then we know the norm of $\z_i$ is at most $\frac{1+\delta}{m}$ (again by RIP property). Now we know
$$|\frac{1}{m}\sum_{i=1}^m n_i \la \A_i,\M\ra|
\le |\frac{1}{m}\sum_{i=1}^m n_i \la \A_i,\M'\ra|+\la\z,\mathbf{n}\ra \le O\left(\sigma \sqrt{\frac{dr\log m}{m}}\|\M\|_F\right).
$$
\end{proof}

\subsection{Matrix Completion}
For matrix completion, we need different concentration inequalities for different kinds of matrices. The first kind of matrix lies in a tangent space and is proved in \cite{candes2012exact}.

\begin{lemma} \cite{candes2012exact} \label{lem:T_mc}
Let subspace 
\begin{equation*}
\mathcal{T} = \{\M \in \R^{d_1 \times d_2} | \M = \U^\star \X\trans + \Y \V^\star {}\trans, \text{~for some~} \X \in \R^{d_1\times r}, \Y\in \R^{d_2\times r}\}.
\end{equation*}
for any $\delta > 0$, as long as sample rate $p \ge \Omega(\frac{\mu r}{ \delta^2 d} \log d)$, we will have:
\begin{equation*}
\norm{\frac{1}{p}\proj_{\mathcal{T}}\proj_{\Omega}\proj_{\mathcal{T}} - \proj_{\mathcal{T}}} \le \delta
\end{equation*}
\end{lemma}

For arbitrary low rank matrix, we use the following lemma which comes from graph theory.


\begin{lemma} \label{lem:graph_mc}
Suppose $\Omega \subset [d_1]\times [d_2]$ is the set of edges of a random bipartite graph with $(d_1, d_2)$ nodes, where any pair of nodes on different side is connected with probability $p$.
Let $d = \max{d_1, d_2}$,then there exists universal constant $c_1, c_2$, for any $\delta>0$ so that if $p \ge c_1  \frac{\log d}{\min\{d_1, d_2\}}$, then with probability at least $1-d^{-4}$, we have for any $\x, \y \in \R^d$:
\begin{equation*}
\frac{1}{p}\sum_{(i, j) \in \Omega} x_i y_j \le \norm{\x}_1\norm{\y}_1 + c_2 \sqrt{\frac{d}{p}}\norm{\x}_2\norm{\y}_2
\end{equation*}
\end{lemma}

\newcommand\J{\mathbf{J}}
\begin{proof}
Let $\A$ be the adjacency matrix of the graph. Clearly $\E[\A] = p \J$ where $\J$ is the all 1's matrix. Let $\mZ = \A - \E[\A]$. The matrix $\mZ$ has independent entries with expectation 0 and variance $p(1-p)$. By random matrix theory, we know when $p \ge c_1\frac{\log d}{\min\{d_1,d_2\}}$, with probability at least $1-d^{-4}$, we have $\norm{\mZ} = \norm{\A-\E[\A]} \le c_2\sqrt{pd}$ \citep{latala2005some}\footnote{The high probability result follows directly from Talagrand's inequality.}. 
Now for any vectors $\x,\y$ simultaneously, we have
\begin{align*}
\frac{1}{p}\sum_{(i, j) \in \Omega} x_i y_j  =& \frac{1}{p} \x^\top \A\y  = \frac{1}{p} \x^\top (p\J+\mZ)\y \\
\le& \la \x,\mathbf{1}\ra\la \y,\mathbf{1}\ra + c_2 \sqrt{\frac{d}{p}}\|\x\|_2\|\y\|_2 \le \|\x\|_1\|\y\|_1 + c_2 \sqrt{\frac{d}{p}}\|\x\|_2\|\y\|_2.
\end{align*}
\end{proof}

Above lemma immediately implies following:

\begin{lemma}\label{lem:Delta_mc}
Let $d = \max\{d_1, d_2\}$. There exists universal constant $c_1, c_2$, for any $\delta>0$ so that if $p \ge c_1  \frac{\log d}{\min\{d_1, d_2\}}$, then with probability at least $1-\frac{1}{2}d^{-4}$, we have for any matrices $\X, \Y \in \R^{d\times r}$:
\begin{equation*}
\frac{1}{p}\norm{\X\Y\trans}^2_{\bar{\Omega}}  \le  \fnorm{\X}^2\fnorm{\Y}^2 + c_2 \sqrt{\frac{d}{p}}\fnorm{\X}\fnorm{\Y}\cdot \max_i\norm{\e_i\trans\X} \cdot \max_j\norm{\e_j\trans\Y}
\end{equation*}
\end{lemma}
\begin{proof}
\begin{equation*}
\frac{1}{p}\norm{\X\Y\trans}^2_{\Omega} = \frac{1}{p}\sum_{(i, j) \in \Omega}\norm{\e_i\trans\X}^2\norm{\e_j\trans\Y}^2
\end{equation*}
The remaining follows from Lemma \ref{lem:graph_mc}.
\end{proof}

On the other hand, for all low-rank matrices we also have the following (which is tighter for incoherent matrices).
\begin{lemma} \cite{ge2016matrix} \label{lem:global_mc}
Let $d = \max\{d_1, d_2\}$, then with at least probability $1-e^{\Omega(d)}$ over random choice of $\Omega$, we have for any rank $2r$ matrices $\A \in \R^{d_1\times d_2}$:
\begin{equation*}
\left|\frac{1}{p} \norm{\proj_{\Omega}(\A)}^2_{\Omega}  - \fnorm{\A}^2\right| 
\le O(\frac{d r\log d}{p}\norm{\A}^2_{\infty}  + \sqrt{\frac{dr\log d}{p}} \fnorm{\A}\norm{\A}_{\infty})
\end{equation*}
\end{lemma}
Although \citet{ge2016matrix} stated the symmetric version, and we need the asymmetric version here, the proof in \citet{ge2016matrix} works directly. In fact, they first proved the asymmetric case in the proof.

Finally, for a matrix with each entry randomly sampled independently with small probability $p$, next lemma says with high probablity, no row can have too many non-zero entries.
\begin{lemma} \label{lem:row_mc}
Let $\Omega_i$ denote the support of $\Omega$ on $i$-th row, let $d = \max\{d_1, d_2\}$. Assume $pd_2\ge \log (2d)$, then with at least probability $1-1/\poly(d)$ over random choice of $\Omega$, we have for all $i\in [d_1]$ simultaneously:
\begin{equation*}
|\Omega_i| \le O(pd_2)
\end{equation*}
\end{lemma}
\begin{proof}
This follows directly from Chernoff bound and union bound.
\end{proof}


\section{Auxiliary Inequalities}
In this section, we provide some frequently used lemmas regarding matrices.
Our first two lemmas lower bound $\norm{\U\U\trans - \Y\Y\trans}_F^2$
by $\norm{(\U-\Y)(\U - \Y)\trans}_F^2$ and $\fnorm{\U - \Y}^2$.
\begin{lemma} \label{lem:aux_delta2}
Let $\U$ and $\Y$ be two $d \times r$ matrices. Further let $\U\trans\Y = \Y\trans \U$ be a PSD matrix. Then,
\begin{equation*}
\norm{(\U-\Y)(\U - \Y)\trans}_F^2 \le 2\norm{\U\U\trans - \Y\Y\trans}_F^2
\end{equation*}
\end{lemma}
\begin{proof}
To prove this, we let $\Delta = \U - \Y$, and expand:
\begin{align*}
\norm{\U\U\trans - \Y\Y\trans}_F^2  =& \norm{\U\Delta\trans + \Delta\U\trans - \Delta\Delta\trans}_F^2 \\
=& \tr(2\U\trans\U\Delta\trans\Delta + (\Delta\trans\Delta)^2 + 2(\U\trans\Delta)^2 - 4\U\trans\Delta\Delta\trans \Delta) \\
=& \tr(2\U\trans (\U - \Delta) \Delta\trans\Delta + (\frac{1}{\sqrt{2}}\Delta\trans\Delta -\sqrt{2}\U\trans\Delta )^2 + \frac{1}{2}(\Delta\trans\Delta)^2) \\
\ge & \tr(2\U\trans\Y\Delta\trans\Delta + \frac{1}{2}(\Delta\trans\Delta)^2) \ge \frac{1}{2}\norm{\Delta\Delta\trans}_F^2
\end{align*}
The last inequality is due to $\U\trans \Y$ is a PSD matrix.
\end{proof}

\begin{lemma} \label{lem:aux_deltalinear}
Let $\U$ and $\Y$ be two $d \times r$ matrices. Further let $\U\trans\Y = \Y\trans \U$ be a PSD matrix. Then,
\begin{equation*}
\sigma_{\min}(\U\trans\U)\fnorm{\U - \Y}^2 \le \fnorm{(\U-\Y)\U\trans}^2 \le \frac{1}{2(\sqrt{2}-1)}\fnorm{\U\U\trans - \Y\Y\trans}^2
\end{equation*}
\end{lemma}
\begin{proof}
The left inequality is basic, we only need to prove right inequality.
To prove this, we let $\Delta = \U - \Y$, and expand:
\begin{align*}
\norm{\U\U\trans - \Y\Y\trans}_F^2  =& \norm{\U\Delta\trans + \Delta\U\trans - \Delta\Delta\trans}_F^2 \\
=& \tr(2\U\trans\U\Delta\trans\Delta + (\Delta\trans\Delta)^2 + 2(\U\trans\Delta)^2 - 4\U\trans\Delta\Delta\trans \Delta) \\
=& \tr((4 - 2\sqrt{2})\U\trans (\U - \Delta) \Delta\trans\Delta + (\Delta\trans\Delta -\sqrt{2}\U\trans\Delta )^2 + 2(\sqrt{2} - 1)\U\trans\U\Delta\trans\Delta ) \\
\ge & \tr((4 - 2\sqrt{2})\U\trans\Y\Delta\trans\Delta + 2(\sqrt{2} - 1)\U\trans\U\Delta\trans\Delta) \ge 2(\sqrt{2} - 1)\fnorm{\U\Delta\trans}^2
\end{align*}
The last inequality is due to $\U\trans \Y$ is a PSD matrix.
\end{proof}

Next we show the difference between matrices formed by swapping sigular spaces of $\M_1$ and $\M_2$ can be upper bounded by the difference between $\M_1$ and $\M_2$.
\begin{lemma}\label{lem:N_to_M_asym}
Let $\M_1, \M_2 \in \R^{d_1\times d_2}$ be two arbitrary matrices whose SVDs are $\U_1 \D_1 \V_1\trans$ and $\U_2\D_2\V_2\trans$. Then we have:
\begin{equation*}
\fnorm{\U_1 \D_1 \U_1\trans - \U_2 \D_2 \U_2\trans}^2 + \fnorm{\V_1 \D_1 \V_1\trans - \V_2 \D_2 \V_2\trans}^2
 \le 2\fnorm{\M_1 - \M_2}^2
\end{equation*}
\end{lemma}

\begin{proof}
Expand the Frobenius Norm out, we have LHS:
\begin{align*}
&\fnorm{\U_1 \D_1 \U_1\trans - \U_2 \D_2 \U_2\trans}^2 + \fnorm{\V_1 \D_1 \V_1\trans - \V_2 \D_2 \V_2\trans}^2 \\
= &2\tr(\D^2_1 + \D_2^2 - \U_1 \D_1 \U_1\trans \U_2 \D_2 \U_2\trans - \V_1 \D_1 \V_1\trans \V_2 \D_2 \V_2\trans)
\end{align*}
On the other hand, we also have RHS:
\begin{align*}
&2\fnorm{\M_1 - \M_2}^2 = 2\fnorm{\U_1 \D_1 \V_1\trans - \U_2 \D_2 \V_2\trans}^2\\
= &2\tr(\D_1^2 + \D_2^2 - \U_1 \D_1 \V_1\trans\V_2 \D_2 \U_2\trans - \U_2 \D_2 \V_2\trans\V_1\D_1\U_1\trans)
\end{align*}

Let $\A = \D_1^{\frac{1}{2}} \U_1\trans \U_2 \D_2^{\frac{1}{2}}$ and $\B = 
\D_1^{\frac{1}{2}} \V_1\trans \V_2 \D_2^{\frac{1}{2}}$. We know to prove the lemma, we only need to show 
$\tr(\A\A\trans + \B\B\trans) \ge \tr(\A\B\trans + \B\A\trans)$. This is true because $\fnorm{\A - \B}^2 \ge 0$, which finishes the proof.
\end{proof}

\end{document}